\documentclass[twoside]{article}

\usepackage[preprint]{aistats2027}

\usepackage[utf8]{inputenc} 
\usepackage[T1]{fontenc}    
\usepackage{hyperref}       
\usepackage{url}            
\usepackage{booktabs}       
\usepackage{amsfonts}       
\usepackage{amsmath}
\usepackage{amsthm}
\usepackage{nicefrac}       
\usepackage{microtype}      
\usepackage{xcolor}         
\usepackage[numbers]{natbib}
\setcitestyle{numbers}
\usepackage{graphicx}
\usepackage{algpseudocode}
\usepackage{algorithm}
\usepackage{subcaption}
\usepackage{tikz}
\usepackage{wrapfig}
\usetikzlibrary{arrows.meta}
\usepackage[rightcaption]{sidecap}
\usepackage{wrapfig}
\usepackage{cleveref}
\usepackage{amssymb}
\usepackage{makecell}

\DeclareMathOperator*{\argmax}{argmax}

\DeclareMathOperator*{\argmin}{argmin}

\DeclareMathOperator*{\expit}{expit}
\DeclareMathOperator*{\logit}{logit}

\newcommand{\bone}{\mathbf{1}}
\DeclareMathOperator{\Law}{Law}
\DeclareMathOperator{\Prob}{Pr}

\newtheorem{theorem}{Theorem}

\newtheorem{proposition}{Proposition}
\newtheorem{assumption}{Assumption}
\newtheorem{definition}{Definition}

\theoremstyle{remark}
\newtheorem{remark}{Remark}

\theoremstyle{definition}
\newtheorem{example}{Example}

\crefname{assumption}{assumption}{assumptions}
\crefname{remark}{remark}{remarks}
\crefname{example}{example}{examples}

\usepackage{xcolor}

\definecolor{LauraColor}{RGB}{180, 40, 40}
\definecolor{MathieuColor}{RGB}{40, 90, 180}
\definecolor{GaelleColor}{RGB}{120, 70, 170}
\definecolor{AntoineColor}{RGB}{20, 130, 90}
\definecolor{UriColor}{RGB}{210, 110, 20}
\definecolor{JulieColor}{RGB}{170, 50, 130}

\begin{document}

 \runningauthor{Laura Fuentes-Vicente, Mathieu Even, Gaëlle Dormion, Antoine Chambaz, Uri Shalit, Julie Josse}

\twocolumn[

\aistatstitle{Set-Valued Policy Learning}

\aistatsauthor{ Laura Fuentes-Vicente \And Mathieu Even \And Gaëlle Dormion }
\aistatsaddress{ Inria PreMeDICaL, Inserm \And Inria PreMeDICaL, Inserm \And Elixir Health}

\aistatsauthor{ Antoine Chambaz \And Uri Shalit \And Julie Josse }
\aistatsaddress{ Université Paris Cité, MAP5 \And Tel-Aviv University \And Inria PreMeDICaL, Inserm }
  
]

\begin{abstract}
  Conventional treatment policies map patient covariates to a single recommended intervention in order to maximize expected clinical outcomes. However, when multiple treatments yield statistically indistinguishable outcomes or when treatment has no effect, recommending a single intervention may result in somewhat arbitrary interventions, undermining clinical adoption and trust. To address this, we propose a set-valued policy learning paradigm. By outputting sets of valuable treatments whose cardinality reflects the recommendation's ambiguity, our approach better supports clinical decision-making. Evaluating a set-valued policy proves subtle due to the range of possible downstream decisions. To do so, we define the set-policy value using a choice function to model clinical decision-making, and we develop doubly robust estimators thereof. Despite its practical importance, set-valued policy learning for categorical treatments remains largely unexplored. In this context, we introduce two complementary approaches: the \textit{Greatest Lower Bound} method, which extends the learning-to-defer framework to multiple treatments, and \textit{conformal set-valued policy learning}, which bridges the gap between unobserved ground-truth optimal treatments and estimated optimal treatment rules. Through experiments on synthetic data and real-world applications to trauma care and in-vitro fertilization (IVF), we demonstrate that our methods produce robust and actionable policies that naturally incorporate clinical considerations while effectively balancing performance and reliability.
\end{abstract}

\section{Introduction}
\label{sec:introduction}
Precision medicine aims to tailor treatment decisions to individual patient characteristics, moving beyond the traditional one-size-fits-all paradigm. This shift has led to a growing body of work in the causal inference literature on learning personalized treatment policies \citep{wager2018estimation, wright2017ranger, sverdrup2020policytree, nordland2022policy}, which output a treatment recommendation to maximize patient outcomes. However, conflicting recommendations from policies trained on the same data and unquantified prediction uncertainty remain major barriers to clinical adoption.
Conflicting recommendations, as illustrated in \Cref{fig:conflict:recom} and described by \citet{bouvier2024machine}, often arise when two or more treatments yield statistically indistinguishable outcomes or when there is no treatment effect. Forcing a single treatment recommendation under these conditions can be not only misleading, but also fundamentally ill-suited. In \Cref{fig:conflict:recom}, treatments $\{1,2\}$ are optimal for the first individual, whereas only $\{3\}$ should be recommended for the fifth.
\begin{figure}[H]
    \centering
   \includegraphics[width=0.8\linewidth]{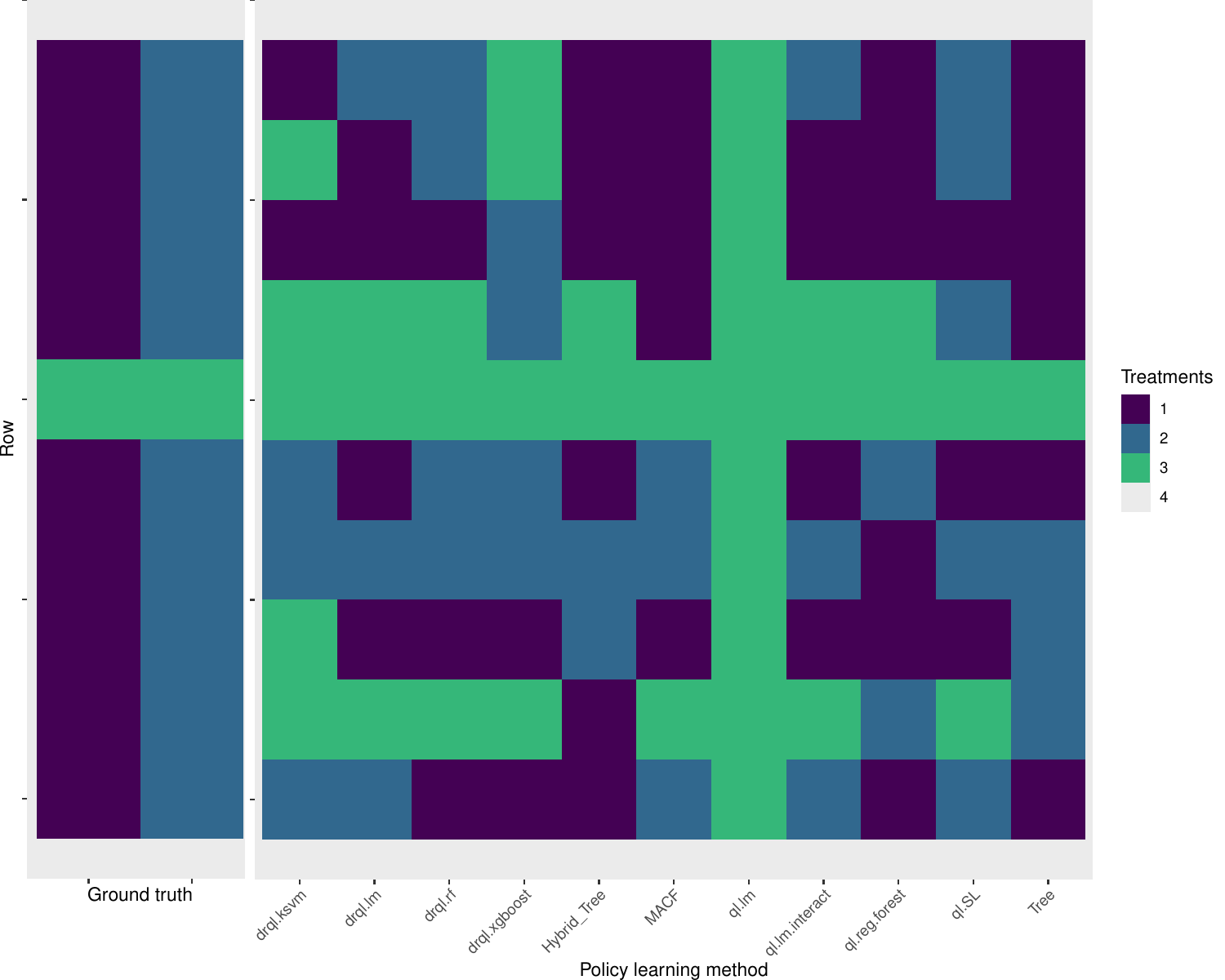}
    \caption{Individual treatment recommendations across policy learning algorithms and ground-truth (synthetic data from Section \ref{subsec:synthetic:data}). Rows represent patients, columns represent algorithms, and colors indicate recommended treatments.}
    \label{fig:conflict:recom}
\end{figure}

To address these limitations, we introduce \textit{set-valued} policies. Rather than systematically issuing a single recommendation, such policies output a subset of valuable treatments, with the cardinality of the set naturally reflecting uncertainty. Our objective is to shift from prescribing a decision to supporting it. 
As in the learning-to-defer paradigm \citep{pmlr-v119-mozannar20b, ghoummaid2024act, gao2021human}, recommendations in the binary case ($\mathcal{A}=\{0,1\}$) either consist of a single treatment (${0}$ or ${1}$) or defer to a downstream decision-maker $(\emptyset, \{0,1\})$. In multi-treatment settings ($|\mathcal{A}|>2$), set-valued policies provide a richer framework than a naive adaptation of learning-to-defer, enabling the deferral of a strict subset of $\mathcal{A}$ rather than relying on an all-or-nothing hand-off. Recommending a set of treatments for each patient reduces decision burden while allowing clinicians to incorporate external constraints and contextual considerations that may not be known a priori and may vary across patients. These include treatment cost \citep{zbMATH07938665, sun2021treatment}, side effects \cite{fuentes2026policy}, patient preferences \citep{MR3777922}, or other clinical considerations, thereby facilitating more flexible and trustworthy decision-making. 

\textbf{Contributions:} We introduce \textit{set-valued policies} (\Cref{sec:problem:setup}), a novel paradigm to support clinical decision-making.
While set-valued policies provide richer, individualized information, the variability they introduce in downstream decisions poses additional challenges for evaluation beyond standard policy value metrics. To that end, we introduce \textit{set-policy values} (\Cref{subsec:set-policy:value}), a new evaluation criterion based on a downstream decision-maker choice function for modeling final clinical decisions, capturing the overall quality of a set-policy. However, estimating these values becomes non-trivial when the choice function depends on the underlying data-generating distribution. To address this, \Cref{sec:spv-estimation} provides doubly robust and semiparametric efficient estimation approaches for set-policy values. We then develop two set-valued policy learning methods: \textit{(i)} \textit{Greatest Lower Bound} (GLB) (\Cref{subsec:greatest-lower-bound}), based on upper- and lower-confidence bounds on conditional outcomes, which achieves conditional coverage; and \textit{(ii)}  \textit{conformal set-valued policy learning} (\Cref{subsec:Conformal policy learning}) which leverages conformal prediction under label noise \citep{einbinderlabelnoise, cauchoisweaksupervision, sesia25adaptative} and incorporates controlled randomness to promote exploration, prevent overconfidence, and reflect optimal treatment unavailability.
 We empirically demonstrate that our methods produce reliable and actionable treatment sets across a range of synthetic experiments, and showcase their practical relevance in two real-world applications: in-vitro fertilization and trauma care (\Cref{sec:numerical:experiments}). To facilitate reproducibility, our code is available at \href{https://anonymous.4open.science/r/setValuedPolicyLearning-6B9C}{\texttt{setValuedPolicyLearning}}. 

\textbf{Related work:} Our framework shares conceptual roots with the learning-to-defer paradigm, in which a trained model and a human expert collaborate to make final decisions. \citet{pmlr-v119-mozannar20b} proposed a method for learning a classifier with an option to defer to an expert. This paradigm was subsequently extended to policy learning \citep{gao2021human}, which, unlike our approach, operates under an all-or-nothing hand-off. The learning-to-defer paradigm was further adapted to guard against unobserved confounding by assuming experts have access to hidden confounders \citep{gao2023confounding, ghoummaid2024act}. In contrast, our framework defers decisions when multiple treatments are statistically indistinguishable (\Cref{subsec:greatest-lower-bound}) or there is high uncertainty, allowing human experts to weigh clinical contextual factors rather than solely compensating for model limitations. 

Presenting filtered option sets to decision-makers has also been explored through conformal prediction (CP) \citep{algolrvovk}. In online settings, these methods yield prediction sets containing the optimal action with high probability, leaving the choice to the decision-maker \citep{babbar2022utilitypredictionsetshumanai, straitouri2023improvingexpertpredictionsconformal, straitouri2023designing, de2024towards, kiyani2025decision}. However, such approaches rely on interactive human feedback, which limits their applicability to our offline setting. Other lines of work have used CP to quantify uncertainty for causal quantities like policy values \citep{taufiq2022confpolicypred, zhang2023confoffpolicypred, dai2020coindice} or the Individual Treatment Effect (ITE) \citep{cicounterfactsandITE, alaa2023confmetalrpredinf,sensitivitycpITE, jonkers2025confconvmc}. While our framework directly handles multiple treatments, ITE interval-based set-valued policies are restricted to binary settings. Such policies can be generated using bound policies \citep{pmlr-v119-mozannar20b, ghoummaid2024act}, which recommend both treatments whenever the interval contains zero. 
\citet{alaa2023confmetalrpredinf} applies CP directly to estimates of the unobserved ITE, a strategy reminiscent of the noisy label literature \citep{einbinderlabelnoise, cauchoisweaksupervision, sesia25adaptative}. Although the method presented in \Cref{subsec:Conformal policy learning} employs a noisy label formulation, our primary approach fundamentally differs: we frame the problem as classification rather than regression. Consequently, we cannot construct a nonconformity score that directly incorporates the target (optimal treatments). In this context, treatment optimality is assessed through potential outcomes, meaning estimation errors propagate into the label evaluation process. This adds significant complexity and directly justifies alternative methodologies, such as randomness injection, to guarantee marginal coverage.

To the best of our knowledge, this work is the first to introduce set-valued policies for uncertainty quantification and formalize their evaluation.

\section{Background and Notations}
\label{sec:background}

We assume access to a sample of $n$ i.i.d. observations $\{\mathcal{D}_{i}=(X_{i},A_{i},Y_{i})\}_{i\in [n]}$ drawn from a distribution $P$. Each observation $\mathcal{D}_{i}$ ($i\in [n]$) comprises: a vector of covariates $X_{i} \in \mathcal{X}\subseteq \mathbb{R}^{d}$, a treatment assignment indicator $A_{i}\in \mathcal{A}=\{1,\ldots, K\}$, with $K>1$, and an outcome $Y_{i}\in \mathcal{Y}\subseteq \mathbb{R}$. By convention, larger values correspond to better outcomes. 
We adopt the potential outcome framework \citep{rubin_potential_outcome} and introduce the tuple $(Y_{i}(1),\ldots,Y_{i}(K))$ of potential outcomes  associated to all treatment levels, with $Y_{i}(a)$  the outcome one would have observed, had patient $i\in[n]$ received treatment $a\in\mathcal{A}$.
We denote $P_{X}$ the marginal law of $X_{i}$ (the same for all $i\in[n]$ under i.i.d.~assumption).
We make the following classical identifiability assumptions.

\begin{assumption}
\label{ass:ci}
For all patients $i\in[n]$, $Y_{i}=Y_{i}(A_{i})$ \emph{(Consistency)};
for all $(x,a)\in\mathcal{X}\times\mathcal{A}$, $P(A_{i}=a|X_{i}=x)>0$ \emph{(Overlap)};
$(Y_{i}(1),\ldots, Y_{i}(K))\perp A_{i}|X_{i}$ \emph{(Unconfoundedness)}.
\end{assumption}

\textbf{Policy learning and its value:}
A policy $\pi$ is an element of the policy class $\Pi = \mathcal{A}^{\mathcal{X}}$. Its performance can be evaluated through its \emph{value},  $\mathbb{E}[Y(\pi(X))]$, interpreted as the average outcome in a world where the policy's recommendations are followed.
Although the policy value is a causal estimand relying on potential outcomes, it can be identified under \Cref{ass:ci} as:
\begin{equation}
\label{eq:policy:value}
\textstyle
     \mathcal{V}_{P}(\pi) = E_{X\sim P_{X}}[\mu(X,\pi(X))]\,,
\end{equation} 
where $\mu(X,a) = E_{P}[Y|X, A=a]$. We define a \textit{value-optimal policy} $\pi^{\star}$ as any maximizer of the mapping $\Pi \ni \pi \mapsto \mathcal{V}_{P}(\pi)$, which then characterizes the $X$-specific set of optimal treatment regimes: 
\begin{equation}\label{eq:pi_star}
\textstyle
    \Pi^{\star}(X)=\{\pi^{\star}(X):\pi^{\star}\in \argmax_{\pi\in \Pi}\mathcal{V}_{P}(\pi) \}\,.
\end{equation} 
Value-optimal policies are typically estimated via two primary paradigms: direct maximization of a consistent estimator of the policy value (\Cref{eq:policy:value}) \citep{qiandirect2011, dudik2011doubly} or indirect maximization of an estimator of the conditional mean outcome $\mu$.
\Cref{sec:problem:setup} generalizes this to set-valued policies.

\section{Set-valued policies}
\label{sec:problem:setup}
\label{subsec:set-policy:value}

We start by formally defining  the notion of \emph{set-valued policies}. 

\begin{definition}[Set-valued policy]
A set-valued policy $C$ is an element of $\mathcal{P}(\mathcal{A})^{\mathcal{X}}$ that maps covariates to a set of treatments. 
\end{definition}

A set-valued policy does not systematically prescribe a single treatment recommendation. Instead, a subset of treatments $\mathcal{C} = C(x)\subseteq \mathcal{A}$ is suggested, to guide the decision-maker's choice. Prescription confidence is naturally encoded through set \emph{cardinality}, where $|\mathcal{C}|=1$ or $|\mathcal{C}|=|\mathcal{A}|$ respectively correspond to high-confidence or no-confidence treatment set recommendations.
Another natural example combines these two extremes.
\begin{example}
    Deferral rules \citep[Learning-to-defer paradigm]{pmlr-v119-mozannar20b} rely on an all-or-nothing hand-off. Given a policy $\pi\in \mathcal{A}^{\mathcal{X}}$ and 
    \textit{deferral rule} $R\in \{0,1\}^\mathcal{X}$, the associated set-valued policy is given by $C_{R,\pi}(x) = \{\pi(x)\}$ if $R=0$ (no deferral to experts) or $C_{R,\pi}(x)=\mathcal{A}$ if $R=1$ (deferral to experts).
\end{example}

Set-policies span the range between these two extremes, by allowing $C(x)$ to be any (potentially strict) subset of treatments, beyond singletons or the full set, allowing to recommend treatments whose non-optimality cannot be significantly enough ruled out with sufficient confidence. In multi-treatment settings, set-valued policies provide a natural generalization of deferral rules and offer a better-suited framework for handling uncertainty.

\textbf{Choice functions and set-policy values.}
The quality of a set-valued recommendation $\mathcal{C}\subseteq \mathcal{A}$ can be measured through the quality of the final treatment recommendation, which is determined jointly by the decision-maker and the set $\mathcal{C}$. We formalize this selection mechanism using a \emph{choice function}. Let $\Law(\mathcal{C})$ be the collection of probability distributions on $\mathcal{C}$.

\begin{definition}[Choice function]
\label{def:choice}
    A \emph{choice function} is a mapping $\delta:(\mathcal{P}(\mathcal{A}),\mathcal{X})\to \Law(\mathcal{A})$, such that for any $\mathcal{C}\subseteq \mathcal A$ and $x\in\mathcal{X}$, $\delta(\mathcal{C}, x) \in \Law(\mathcal{C})$.
\end{definition}

\begin{definition}[Set-policy value, SPV]
    Let $\delta$ be a choice function (\Cref{def:choice}) and $C\in\mathcal{P}(\mathcal{A})^{\mathcal{X}}$ a set-valued policy. 
   We define its $\delta$-specific set-policy value (SPV) as $\mathbb{E}_{A|X\sim \delta(C(X), X)}\left[Y(A) \right]$. Under \Cref{ass:ci}, the SPV is identified as:
\begin{equation}\label{eq:id_SPV}
    \mathcal{V}_P(C;\delta) = E_{(A,X)\sim \delta_{C}P_{X}}\left[ \mu(X, A) \right],
\end{equation} where $\delta_{C}P_{X}$ is the joint law of $(A,X)$ when $X\sim P_{X}$ and $A|X \sim \delta(C(X),X)$.
\end{definition} 
The true observed value is thus $\mathcal{V}_\star(C) = \mathcal{V}_P(C;\delta_\star)$ where $\delta_\star$ is the experts' true choice function. Since $\delta_\star$ depends on experts' final decisions, estimating $\mathcal{V}_\star$ and $\delta_\star$ in observational settings would require a performative framework \citep{perdomo2020performative}. Assessing the value of a set-valued policy thus proves more subtle and challenging than in the standard policy learning setting. 

\textbf{Proxy choice functions.}
To address the unavailability of $\delta_{\star}$, we propose using \textit{proxy} choice functions. A first and natural proxy is the \textit{uniform choice function}, given by $\delta_\mathrm{unif}(\mathcal{C},x)=\mathcal{U}(\mathcal{C})$ which samples a treatment uniformly at random amongst $\mathcal{C}$, and therefore penalizes the inclusion of treatments with lower expected outcomes. The \textit{uniform SPV} provably lower bounds the true SPV $\mathcal{V}_\star$, provided experts outperform random choice (see \Cref{app:proof:lower:bound:SPV}). Clinical guidelines like dose de-escalation can sometimes define a choice function. For example, $\delta_{\mathrm{lower}}$ selects the minimum dose from candidate set $\mathcal{C}$ providing a risk-averse SPV (\Cref{app:more:fig:elixir}). In most scenarios, however, such strategies are unknown a priori.

We next propose the \textit{propensity SPV} which leverages observational data to capture clinical prior knowledge under the assumption that decision-makers prescribe treatments consistently with historical data. The propensity SPV complements the uniform SPV, as it lowers the importance of recommended treatments that would not be chosen by clinicians.

\begin{definition}[Propensity choice function]
\label{def:propensity_SPV}
    Let $\delta_{P}$ be the choice function associated with the propensity score $e$, for any $a \in\mathcal{A}$ and fixed $\mathcal{C}\subseteq\mathcal{A}$, such that,
    \begin{equation*}
     \Prob_{A\sim \delta_{P}(\mathcal{C},x)}\left(A = a \right) = \frac{\bone{\{a\in\mathcal{C}\}}e(a|x)}{\sum_{a'\in \mathcal{C}}e(a'|x)}\,,
    \end{equation*}
    where $e(a|x)=P(A=a|X=x)$. \vspace{-0.2cm}
\end{definition} 
The propensity choice function $\delta_{P}$ samples treatment according to the propensity score, conditioned on being in the prescribed set. 
For brevity, we denote the propensity SPV $\mathcal{V}_P(\cdot;\delta_{P})$ as $\mathcal{V}_{e}$, which satisfies:
\begin{equation}\label{eq:id_prop_SPV}
   \mathcal{V}_{e}(C) = E_{X\sim P_X}\left[ \sum_{a\in \mathcal{C}} \frac{e(a|X)\mu(X,a) }{\sum_{a'\in \mathcal{C}} e(a'|X)} \right]\,.
\end{equation} 
If historical prescriptions closely reflect future decision-maker behavior, $\mathcal{V}_{e}$ will closely match $\mathcal{V}_\star$. 

In most of our experiments, we provide both propensity and uniform SPVs as they offer complementary information. A large uniform SPV value indicates that the set yields predominantly high-performing treatments, even when set cardinality is large. Conversely, a large propensity SPV value indicates that the set effectively enables the clinician to make an optimal choice.

\textbf{Coverage of a set-valued policy.}
For a level $\alpha\in(0,1)$, a set-policy $C$ satisfies marginal or conditional coverage if, respectively:
\begin{align}
    \label{eq:marginal:cov:A}
    &\Prob(\Pi^{\star}(X_{n+1}) \cap C(X_{n+1}) \neq \emptyset) \geq 1-\alpha\,, \text{or}\\
    &\Prob(\Pi^{\star}(X) \cap C(X) \neq \emptyset | X=X_{n+1}) \geq 1-\alpha\,. \nonumber
\end{align}
Coverage guarantees that, with probability at least $1-\alpha$, the set contains at least one optimal treatment. Consequently, coverage can hold even when the set includes treatments yielding lower expected outcomes, which could hinder decision-making. Moreover, coverage does not reflect closeness to the true optimal set $\Pi^{\star}(X_{i})$, and is difficult to guarantee in our setting (see \Cref{subsec:Conformal policy learning}).

\section{Set-Policy Value Estimation}
\label{sec:spv-estimation}
We now describe how to estimate the SPV introduced in \Cref{sec:problem:setup} from observational data. Throughout this section, the set-valued policy $C$ is fixed. We assume that $C(x)$ is nonempty for all $x \in \mathcal{X}$, as empty recommendations can be interpreted as full deferral by replacing them with $\mathcal A$ before evaluation.

\textbf{Fixed choice functions.}
When the choice function $\delta$ is independent of the observational distribution $P$, the SPV $\mathcal{V}_{P}(C;\delta)$ reduces to the standard policy value $\mathcal{V}_{P}(\pi_{C,\delta})$ (\Cref{eq:policy:value}) where $\pi_{C,\delta}$ denote treatments sampled from $x\mapsto \delta(C(x),x)$. Under \Cref{ass:ci}, the SPV admits standard outcome-regression and inverse-probability-weighted representations. This enables the construction of G-computation, IPW, AIPW, and TMLE estimators for the SPV, as detailed in \Cref{app:evaluation_fixed_choice_function}, where we derive its corresponding efficient influence function (EIF). Uniform and minimum-dose choice functions fit within this setup.

\textbf{Distribution-dependent choice functions.}
A different estimation problem arises when the choice function depends on $P$, as in the case of the propensity SPV (\Cref{eq:id_prop_SPV}). The target quantity is expressed as a functional $P\mapsto\mathcal V_{P}(C,\delta_P)$, where perturbations in the observational distribution affect both the conditional mean outcome and choice function. Such targets belong to the framework of distribution-dependent stochastic interventions \citep{kennedy2019nonparametric}. 
Whenever differentiation is justified, the EIF is given by: 
\begin{proposition}[Propensity SPV EIF]\label{prop:eif_dependent}
 Let $\mathcal{D}=(X,A,Y)$, $m(x)= E[\mu(X,A)|X=x, A\in C(X)]$ and $g_{C}:x\mapsto \sum_{a\in C(X)} e(a|X)$. 
    The \emph{EIF} of the propensity SPV, decomposes into:
    \begin{equation*}
        \phi(P)(\mathcal{D}) = m(X) + \frac{\bone{\{A\in C(X)\}}}{g_{C}(X)} \left(Y -m(X) \right) -\mathcal{V}_{e}(C).
    \end{equation*}
\end{proposition} 
Next, we show that the AIPW estimator is doubly robust and asymptotically normal (similar results hold for the TMLE estimator, see \Cref{app:SPV_distrib_dependent_choice}). We focus for simplicity on the case where  the nuisance parameters $m$ and $g_{C}$ are ``learned on independent data'' (i.e. estimated on another sample of $n$ i.i.d.\ observations independently drawn from $P$). 
Denote  $\|\cdot\|_2$ the $L_{2}(P)$-norm. 
\begin{theorem}[Double Robustness]\label{thm:double_robustness}
    Let $\hat m,\hat g_C$ be estimators of $m$ and $g_C$, learned on independent data such that $\|\hat m-m_{P}\|_2\to0$, $\|\hat g_{C}-g_{C,P}\|_2\to0$ and $g_{C}\geq\eta$, $\hat g_{C}\geq\eta$ for some $\eta>0$. Under \Cref{ass:ci}, the AIPW estimator, $\hat{\mathcal{V}}_e^\mathrm{AIPW}(C)$, is given by:
    \begin{equation*}
    \frac{1}{n}\sum_{i=1}^n \hat m(X_{i}) +\frac{\bone\{A_{i}\in C(X_{i})\}}{\hat g_C(X_{i})} \big(Y_{i} - \hat m(X_{i})\big).
    \end{equation*}
    It satisfies:
    \begin{align*}\label{eq:double_robustness}
        &\sqrt{n}\big(\hat{\mathcal{V}}_e^\mathrm{AIPW}(C) - \mathcal{V}_e(C)\big) \\
        &+ O_{P}\left(\sqrt{n}\eta^{-1}\|\hat m-m\|_2\|\hat g_C-g_C\|_2\right) \leadsto \mathcal{N}\left(0,\sigma_e^2(C)\right)
    \end{align*} with
    \begin{equation*}
       \sigma_e^2(C) = \mathrm{Var}_{P}(m(X)) + E_{P}\left[\frac{\mathrm{Var}_{P}\left(Y|X,A\in C(X)\right)}{g_C(X)}\right]\,.
    \end{equation*}
\end{theorem}
The AIPW estimator is asymptotically normal whenever $\|\hat m-m\|_2\|\hat g_C-g_C\|_2 = o_{P}(1/\sqrt{n})$. Nuisance estimators $\hat m,\hat g_C$ can be obtained based on a propensity score estimator $\hat{e}$ and a conditional mean estimator $\hat{\mu}$. 
However, the above condition on $\hat{m}$ and $\hat{g}_{C}$ is milder than its counterpart $\|\hat \mu-\mu\|_2\|\hat e-e\|_2 = o_{P}(1/\sqrt{n})$. In practice, instead of learning $\hat{m}$ and $\hat{g}_{C}$ on independent data, cross-fitting can be used, yielding similar results. Finally, since $\sigma_{e}^{2}(C)$ is the variance of the EIF of $\mathcal{V}_{e}(C)$, the AIPW estimator attains the nonparametric efficiency bound under our assumption, and asymptotic confidence intervals can be built using  a consistent estimator of this variance.

\section{Set-valued policy learning}
\label{sec:set-valued-policy-learning}
\subsection{Greatest lower bound method (GLB)}
\label{subsec:greatest-lower-bound}

We now introduce our first set-valued policy learning approach, which hinges on uncertainty quantification relative to estimators of $\mu$. 
This approach implicitly assumes reliable lower- and upper-confidence bounds on the conditional means. This is the case for well-specified generalized linear models, for which asymptotic distributions of the parameters are known, or more generally for random forest methods which quantify point-wise uncertainty using the infinitesimal jackknife \citep{tibshirani2018package}.
Fix $\alpha\in (0,1)$ and let $(x,a) \mapsto \hat{\ell}_\alpha\left(x,a\right)$ and $(x,a)\mapsto \hat{u}_\alpha\left(x,a\right)$ denote available lower- and upper-bound estimators at level $1-\alpha/2$, for any $(x,a)\in \mathcal{X}\times \mathcal{A}$ such that,
 \begin{equation*}
     \Prob\left(\mu(x,a)\in \left[\hat{\ell}_\alpha\left(x,a\right), \hat{u}_\alpha\left(x,a\right)\right]\right)\geq 1-\alpha\,. 
 \end{equation*} 
We define the \textit{greatest lower-bound treatment policy} $\underline{a}\in\mathcal{A}^{\mathcal{X}}$  as
\begin{equation*}
    x \mapsto \underline{a}(x) \in \argmax\{\hat{\ell}_\alpha\left(x,a\right):a\in \mathcal{A}\}\,,
\end{equation*} and construct the set-valued policy by including all treatments whose upper-bound exceeds the greatest lower-bound benchmark, yielding for $x\in \mathcal{X}$:
\begin{equation}\label{eq:greatest_lower_bound_set_policy}
    \hat{C}_{\alpha}(x)=\{a\in \mathcal{A}: 
    \hat{u}_\alpha\left(x,a\right)\geq \hat{\ell}_\alpha\left(x,\underline{a}(x)\right)\}\,. 
\end{equation}
Note that $x\mapsto \underline{a}(x)$ may need to be estimated on a separate fold from the one used to estimate the lower- and upper-bounds to safeguard against overfitting.

\begin{proposition}
\label{prop:naive}
It holds $P$-almost surely that: 
\begin{equation*}
    \Prob(\Pi^{\star}(X) \cap \hat{C}_{\alpha}(X) \neq \emptyset|X=X_{n+1})\geq 1-\alpha\,.
\end{equation*}
\end{proposition} 
The proof can be found in \Cref{app:proof:naive}. GLB thus satisfies \emph{conditional} coverage, which implies \emph{marginal} coverage.
The \emph{greatest lower bound} method generalizes the \emph{bounds policy} approach in learning-to-defer \citep{pmlr-v119-mozannar20b, ghoummaid2024act} by extending it to multiple treatments. This further establishes set-valued policy learning as the natural extension of learning-to-defer beyond binary treatment regimes. Beyond its theoretical alignment, the method offers significant interpretability by retaining treatments whose expected potential outcomes are not significantly worse than the best.

The main challenge of this approach lies in the non-trivial nature and scarcity of methods for constructing valid lower- and upper-bounds for~$\mu$. 
To address this limitation, the following section introduces an alternative set-valued policy learning framework based on conformal prediction under label noise.

\subsection{Conformal set-valued policy learning}
\label{subsec:Conformal policy learning}

Policy learning can be framed as a classification task, naturally fitting into the conformal prediction framework (see \Cref{app:cp}). Consider $n$ i.i.d.\ samples $\{\mathcal{O}_{i}=(X_{i},A_{i}^{*})\}_{i\in [n]}$, where each $A_{i}^{*} \in \Pi^{\star}(X_{i})$ corresponds to a draw from the uniform distribution over the $i$-specific optimal treatments set $\Pi^\star(X_{i})$ (conditionally on $X_{i}$). 
Applying conformal classification algorithms for policy learning is not straightforward in practice due to the unavailability of the optimal treatments as labels. To address this, we introduce an intermediate label estimation step based on $\{\mathcal{D}_{i}=(X_{i},A_{i},Y_{i})\}_{i\in [n]}$, framing the problem within the noisy label regime \cite{einbinderlabelnoise,sesia25adaptative}. The following paragraphs outline each step of the conformal set-valued policy learning algorithm.

\textbf{Conformal policy learning meta-algorithm:} 
Partition  $\{\mathcal{D}_{i}=(X_{i},A_{i},Y_{i})\}_{i\in [n]}$ into three mutually disjoint subsets indexed by $\mathcal{I}_{b}, \mathcal{I}_{\mathrm{train}}$ and $\mathcal{I}_{\mathrm{cal}}$. Let  $n_{\mathrm{cal}}$ denote the cardinality of $\mathcal{I}_{\mathrm{cal}}$.  Our meta-algorithm proceeds as follows.
\begin{enumerate}
    \item \emph{Black-box label generation.}
    Train a label generation algorithm $\mathcal{B}$ using $\mathcal{D}_{b} = \{\mathcal{D}_{i}:i \in \mathcal{I}_{b}\}$. This yields $\mathcal{B}(\mathcal{D}_{b})\in \mathcal{A}^{\mathcal{X}}$, which maps covariates to treatment decisions (i.e.\ labels). For $x\in\mathcal{X}$, $\mathcal{B}(\mathcal{D}_b)(x)$ should be understood as a black-box estimate of an element of $\Pi^{\star}(x)$.
    
    \item \emph{Nonconformity score.}
    Learn a nonconformity score function $\hat{s}:\mathcal{X}\times \mathcal{A}\to \mathbb{R}$ using the training subset $\{\mathcal{D}_{i}:i \in \mathcal{I}_{\mathrm{train}}\}$. In the next step,  $\hat{s}$ quantifies alignment of any candidate sample  $(x,a)$ relative to calibration data. 
    
    \item \emph{Noisy calibration.} Generate the noisy observations $\{\hat{\mathcal{O}}_{i}= (X_{i},\hat{A}_{i}^{*})\}_{i\in \mathcal{I}_{\mathrm{cal}}}$ with noisy labels defined as $\hat{A}^{*}_{i}=\mathcal{B}(\mathcal{D}_{b})(X_{i})$ (for all $i\in \mathcal{I}_{\mathrm{cal}}$). Compute the nonconformity scores $\{\hat{S}_{i}=\hat{s}(X_{i},\hat{A}^{*}_{i})\}_{i\in \mathcal{I}_{\mathrm{cal}}}$, and derive their $(1-\alpha)$-quantile, $\hat{q}_{1-\alpha}$. Finally, define the set-valued policy as
    \vspace{-5pt}
\begin{equation}\label{eq:noisy_policy_set}
    \textstyle
        x \mapsto  \hat{C}_{\alpha}(x)=\{a\in \mathcal{A}: \hat{s}(x,a) \leq \hat{q}_{1-\alpha}\}\,.
    \end{equation} 
\end{enumerate}
\textbf{1. Black-box algorithm $\mathcal{B}$:}
Any algorithm $\mathcal{B}$ can be used to generate the noisy labels. Examples include, \textit{(i)} any fixed policy learning algorithm which yields estimates of optimal treatment (OTR) (e.g. Q-learning, policy trees, etc.); \textit{(ii)} a foundation model yielding treatment recommendations; \textit{(iii)} any weighted mixture of different OTRs. 

\textbf{2. Nonconformity score:} The nonconformity score function plays a central role in quantifying similarity between a candidate sample $(x,a)\in \mathcal{X}\times \mathcal{A}$ and the calibration data $\{(X_{i},\hat{A}_{i}^{*})\}_{i\in \mathcal{I}_{\mathrm{{cal}}}}$, therefore capturing how frequently treatment $a$ is ``optimal'' for similar covariates. An optimal treatment $a^{*}$ can be defined as a maximizer of $a \mapsto \mu(x,a )$. Therefore, a well-designed nonconformity score should reflect the distance between the conditional mean outcome under a candidate treatment $a\in \mathcal{A}$, $\mu(x, a)$, and that under the optimal treatment, $\mu(x, a^{*}) = \max_{a \in \mathcal{A}} \mu(x,a)$. It is crucial to avoid predominantly capturing inter-individual heterogeneity rather than within-individual variation in treatment response. Motivated by the \emph{margin score} in conformal classification, we define the oracle nonconformity score $s^*(x,a) = \max_{a'\neq a}\mu(x,a') - \mu(x,a)$ for $(x,a)\in\mathcal{X}\times\mathcal{A}$. Its empirical counterpart is defined by plugging in an estimator $\hat{\mu}$ trained on $\{\mathcal{D}_{i}: i \in \mathcal{I}_{\mathrm{train}}\}$:
\begin{equation}
\label{eq:non:conformity:score:reg}
\hat{s}(x,a) = \max_{a'\neq a}\hat{\mu}(x,a') - \hat{\mu}(x,a).
\end{equation}
\textbf{3. Noisy calibration:} The (estimated) noisy labels $\hat{A}_{i}^{*}$ are used as proxies for the unobserved ground truth $A_{i}^{*}\in\Pi^{\star}(X_{i})$, for all $i\in \mathcal{I}_{\mathrm{cal}}$. For any nonconformity score function  $s$, the empirical distribution of the resulting noisy scores $\{\hat{S}_{i}=s(X_{i},\hat{A}^{*}_{i})\}_{i\in \mathcal{I}_{\mathrm{cal}}}$ can differ significantly from that of the true scores $\{S_{i}=s(X_{i},A^{*}_{i})\}_{i\in \mathcal{I}_{\mathrm{cal}}}$. This discrepancy compromises coverage guarantees, particularly when the $(1-\alpha)$-quantile under noisy labels underestimates the true quantile.

The next paragraph introduces a procedure that yields marginal coverage under additional assumptions.

\textbf{Towards marginal coverage guarantees.}
Previous work \cite{alaa2023confmetalrpredinf, sesia25adaptative} establishes marginal coverage for conformal prediction methods under label noise, relying on specific model assumptions and stochastic dominance conditions, as detailed in \Cref{app:noisy_ICP}. To reduce trust in the noisy labels $\hat{A}^{*}_{i}= \mathcal{B}(\mathcal{D}_{b})(X_{i})$ and satisfy these dominance conditions, we propose a label perturbation procedure that injects controlled randomness directly into the predictions. This exploration mechanism inflates the nonconformity scores, promoting conservative score estimates and enhancing coverage.

Let $\{A^{\mathrm{rd}}_{i}\}_{i \in [n]}$ be sampled independently and uniformly from $[K]$. Fix $r\in [0,1]$ and let $\{R_{i}\}_{i \in [n]}$ be an i.i.d.\ sample from the Bernoulli distribution with success probability $r$. For each $i \in [n]$, define the $i$-specific perturbed label as 
\vspace{-5pt}
\begin{equation}
\label{eq:noisy:label:generation}
\textstyle
    \hat{A}^{*}_{r,i} = R_{i}\cdot A^{\mathrm{rd}}_{i} + (1-R_{i})\cdot \hat{A}^{*}_{i}\,.
\end{equation} 
Perturbed labels play a pivotal role in \Cref{algo:CSP}.

\begin{algorithm}[h]
\caption{Conformal set-valued policy learning}
\label{algo:CSP}
\begin{algorithmic}
\Require Data $\{\mathcal{D}_{i}=(X_{i},A_{i},Y_{i})\}_{i\in [n+1]}$, \\ level $\alpha \in [0,1]$, randomness parameter $r \in [0,1]$.
\State Partition data into subsets indexed by  $\mathcal{I}_{b}, \mathcal{I}_{\mathrm{train}}$ and $\mathcal{I}_{\mathrm{cal}}$. Let $n_{\mathrm{cal}}$ be the cardinality of  $\mathcal{I}_{\mathrm{cal}}$.
\State Train black-box label generation algorithm $\mathcal{B}$ on $\mathcal{D}_{b}=\{\mathcal{D}_{i}:i\in \mathcal{I}_{b}\}$.
\State Learn nonconformity score function $\hat{s}$ (i.e.\ $\hat{\mu}$) \\
(\Cref{eq:non:conformity:score:reg}) using $\{\mathcal{D}_{i}:i\in \mathcal{I}_{\mathrm{train}}\}$.
\For{$i \in \mathcal{I}_{\mathrm{cal}}$}
    \State $A_{i}^{\mathrm{rd}} \sim \mathcal{U}([K])$, $\hat{A}^{*}_{i} = \mathcal{B}(\mathcal{D}_{b})(X_{i})$, $R_{i} \sim \mathrm{Ber}(r)$,
    \State $\hat{S}_{i} = \hat{s}(X_{i}, \hat{A}^{*}_{r,i})$\\ 
    where $\hat{A}^{*}_{r,i} = R_{i}\cdot A_{i}^{\mathrm{rd}} + (1-R_{i}) \cdot \hat{A}_{i}^{*}$ (\Cref{eq:noisy:label:generation}).
\EndFor
\State Define $\hat{q}_{1-\alpha}^{r} =\mathrm{Quantile}\left(\frac{\lceil(1-\alpha)(n_{\mathrm{cal}}+1)\rceil}{n_{\mathrm{cal}}}; \{\hat{S}_{i}\}_{i\in \mathcal{I}_{\mathrm{cal}}}\right)$.
\State \Return $\hat{C}^{r}_{\alpha}(X_{n+1}) = \{a \in \mathcal{A} : \hat{s}(X_{n+1},a) \leq \hat{q}_{1-\alpha}^{r}\}$.
\end{algorithmic}
\end{algorithm}
Assuming that the random labels ($A^{\mathrm{rd}}_{i}$) yield higher nonconformity scores than the black-box labels ($\hat A_i^*$) and optimal labels ($A_i^*$), 
the cardinality of the recommended treatment set is nondecreasing in $r$. This is formalized in the next proposition, where $F_{\mathrm{rd}}$, $\hat{F}$ and $F$ denote the corresponding nonconformity score CDFs.
\begin{proposition}
\label{prop:set_cardinality}
    If $F_\mathrm{rd}(\cdot)\leq \hat F(\cdot)$, 
    then $\hat C^{r}_\alpha (X_{n+1}) \subseteq \hat C^{r'}_\alpha (X_{n+1})$ (as defined in \Cref{algo:CSP}) for any $0 \leq r \leq r' \leq 1$.
\end{proposition}
See proof in \Cref{proof:prop:cardinality}. The following theorem establishes that marginal coverage is guaranteed whenever the randomness level $r$ is sufficiently large for the perturbed nonconformity scores to exceed those of the true optimal (illustrated in \Cref{app:intuition}: \Cref{fig:sketch,fig:sketch:noise}).
\begin{theorem}
\label{thm:coverage:r}
 If $F_\mathrm{rd}(\cdot)\leq \hat F(\cdot)$ and $F_\mathrm{rd}(\cdot)\leq  F(\cdot)$, and if the perturbed nonconformity score CDFs are continuous at the $(1-\alpha)$-quantile, then the  confidence set $\hat{C}^{r}_{\alpha}$ defined in \Cref{algo:CSP} achieves marginal coverage at level $(1-\alpha)$ whenever $r$ satisfies $r \in (\bar{r},1] \cup \{1\}$,  with 
\vspace{-5pt}
  \begin{equation*}
      \bar{r} = \inf \left\{r \in [0,1]: E[F(\hat{q}_{1-\alpha}^{r})]\geq 1-\alpha \right\}\,.
  \end{equation*}
\end{theorem}
 The proof can be found in \Cref{app:proof:thm}. 
Since  $\bar{r}$ depends on unknown quantities, it cannot be estimated directly.
In practice,  $r$ can instead be tuned to maximize the set-policy value.

\begin{remark}
\textit{(i)} Empirically, aggregating OTRs reduces reliance on randomness injection for marginal coverage (\Cref{tab:exact_match,tab:exact_match:tree}: Aggregation).

\textit{(ii)} An alternative to randomness injection requires rank-preserving nonconformity scores (\Cref{app:rank:preserving}), an assumption that is difficult to verify in practice.
\end{remark}

\section{Numerical experiments}
\label{sec:numerical:experiments}
\subsection{Synthetic setting}
\label{subsec:synthetic:data}
\textbf{Data simulation:} We simulate observational data with Gaussian covariates $X\in \mathbb{R}^{5}$ and $K=4$ treatment levels. Covariate space is divided linearly in two: in the first half, potential outcomes under treatments $\{1,2\}$ yield larger outcome values, while in the other half the optimal treatment is  $\{3\}$. \Cref{app:synthetic:data} details the synthetic data-generating process (DGP). 

We generate $n=10{,}000$ training samples and $n_{\mathrm{test}}= 5{,}000$ test samples. We evaluate eleven baseline policies, such as \texttt{policytree} \citep{sverdrup2020policytree} and Q-learning with several base learners (see \Cref{subfig:naive:policies}). We estimate the upper- and lower-bounds for GLB using both a linear model with interactions \cite{nelder1972generalized} and a regression forest \cite{athey2019generalized,tibshirani2018package}. For the conformal procedure we set $r\in \{0, 0.1\}$, fit nonconformity scores via Super Learner \citep{van2007super, SuperLearner} and generate noisy labels using several baseline policies. For both GLB and the conformal procedure we fix $\alpha=0.1$. We report results across 30 bootstrap replicates for a representative subset of methods (Complete results: \Cref{fig:spv:boxplot,tab:exact_match}). Our evaluation leverages oracle data to determine true set-policy values and assess metrics outlined in \Cref{def:coverage_notions} (e.g.\ exact match, strict coverage etc.). The following patterns were consistent across the other DGP (\Cref{app:additional:setting}: \Cref{fig:spv:boxplot:tree,tab:exact_match:tree}). 

\begin{figure}[h]
    \centering
    \includegraphics[width=\linewidth]{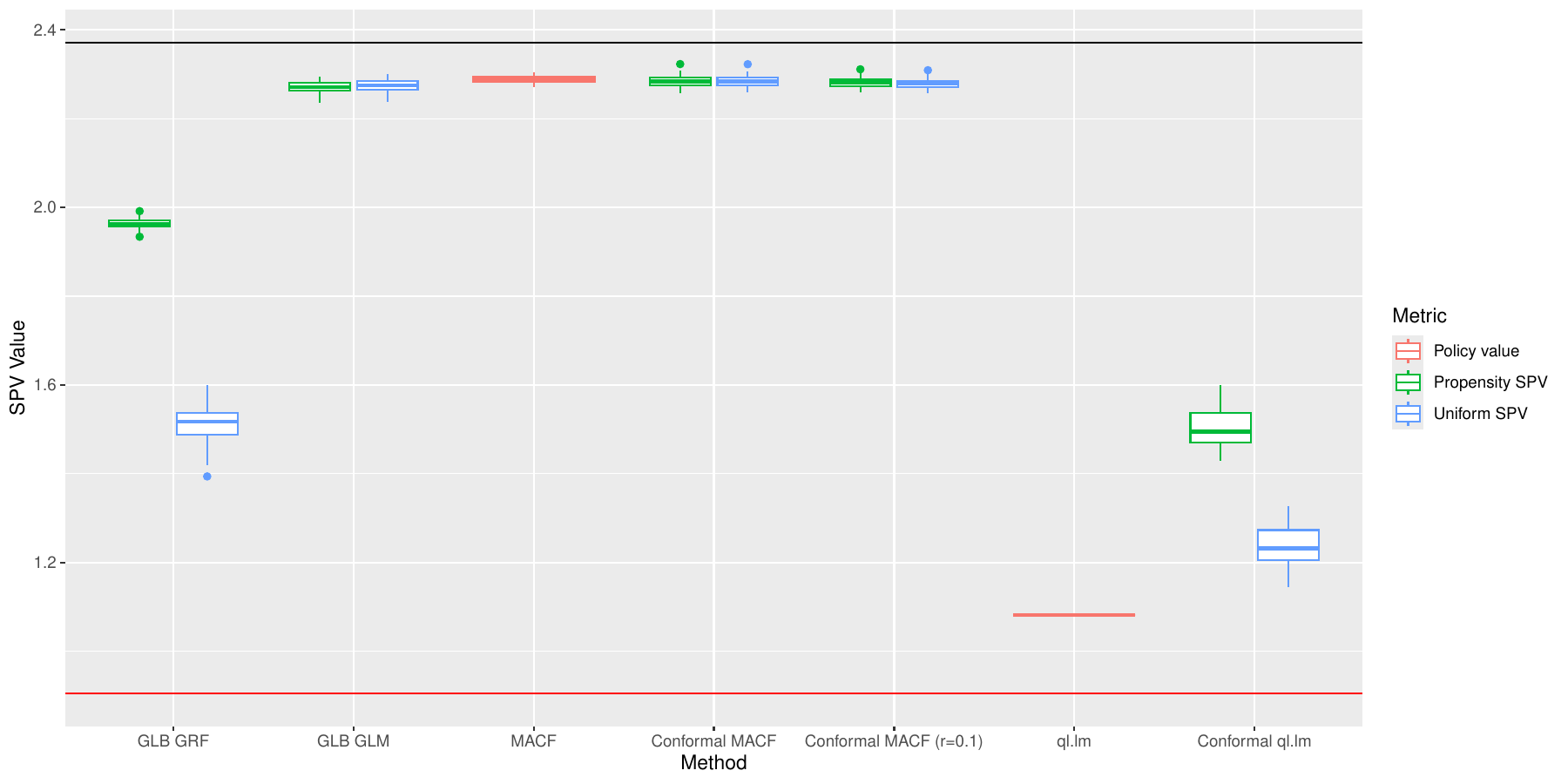}
    \caption{Boxplots comparing uniform (blue) and propensity-based (green) SPVs for different set-valued policies against the values of different baseline policies (red). Horizontal lines indicate optimal policy (black) and random treatment allocation (red) values.}
    \label{fig:spv:boxplot:reduced}
\end{figure}
\textbf{Equivalent policy value with set-valued flexibility:} As shown in \Cref{fig:spv:boxplot:reduced}, some baseline policies like \texttt{MACF} are close to the oracle optimum (in black). Crucially, conformal set-valued policies derived from these baseline policies maintain comparable value while providing quantified uncertainty, giving decision-makers flexible options without sacrificing performance. Additionally, in this setting, GLB GLM achieves performance closely matching all these approaches. 

\textbf{Robustness against suboptimal baselines:}
Conversely, when using a lower-performing baseline (\Cref{fig:spv:boxplot:reduced}: ql.lm), the conformalized set-valued policy actively improves the overall values. 
\begin{table}[t]
\centering
\small
\setlength{\tabcolsep}{1pt}
\begin{table}[H]
\centering\begingroup\fontsize{9}{11}\selectfont
\begin{tabular}{lccccccc}
\toprule
\multicolumn{1}{c}{ } & \multicolumn{2}{c}{\textbf{GLB}} & \multicolumn{3}{c}{\textbf{MACF}} & \multicolumn{2}{c}{\textbf{ql.lm}} \\
\cmidrule(l{3pt}r{3pt}){2-3} \cmidrule(l{3pt}r{3pt}){4-6} \cmidrule(l{3pt}r{3pt}){7-8}
\textbf{Set Type} & GRF & GLM & Pol. & Conf. &  Conf.$_{0.1}$ & Pol. & Conf.\\
\midrule
$\Pi^{\star}(X_{i})$ = \{1, 2\} & 0.31 & 0.81 & \textbf{0.00} & \textbf{0.70} & \textbf{0.84} & 0.00 & 0.00\\
$\Pi^{\star}(X_{i})$ = \{3\} & 0.18 & 1.00 & 0.99 & 0.99 & 0.97 & 1.00 & 0.29\\
Mean card. & \textbf{2.64} & 1.62 & 1.00 & 1.49 & 1.63 & 1.00 & 3.45\\
Strict cov. \{1, 2\} & 1.00 & 0.87 & \textbf{0.00}  & \textbf{0.72} & \textbf{0.89} & \textbf{0.00}  & \textbf{1.00}\\
Strict cov. \{3\} & 1.00 & 1.00 & 0.99 & 1.00 & 1.00 & 1.00 & 1.00\\
Total strict cov. & \textbf{1.00} &\textbf{0.92} & 0.35 & 0.82 & \textbf{0.93} & 0.35 & \textbf{1.00}\\
Marginal cov. & \textbf{1.00} & \textbf{0.97} & 0.96 & \textbf{0.97} & \textbf{0.98} & 0.35 & \textbf{1.00}\\
\bottomrule
\end{tabular}
\endgroup{}
\end{table}
\vspace{-0.5cm}
\caption{Comparison of GLB, baseline (Pol.) and Conformal (Conf.) policies ($r\in \{0,0.1\}$). Metrics include conditional exact match, mean cardinality, conditional strict coverage and total strict and marginal coverage.}
\label{tab:reduced}
\vspace{-0.3cm}
\end{table}

\textbf{Treatment indistinguishability and coverage:} 
Baseline policies fail to capture the indistinguishability between treatments $\{1, 2\}$ (\Cref{tab:reduced}: Pol. \texttt{MACF} and ql.lm). In contrast, both GLB and conformal set-valued policies increase conditional exact match rates ($C=\Pi^{\star}$, \Cref{tab:reduced}: rows 1-2) and achieve marginal coverage regardless of $r$ (\Cref{tab:reduced}: row 7). For stricter coverage criteria such as $\Pi^{\star}\subseteq C$ (\Cref{tab:reduced}: rows 4-6), GLB GRF trades informativeness for higher rates. Meanwhile, conformal variants based on top-performing baselines meet this standard only when $r = 0.1$. 

\textbf{Propensity improves upon uniform set-policy values:} 
Propensity SPV values are consistently higher than uniform SPV values for the same set-valued policy, as the latter heavily penalizes the inclusion of suboptimal treatments (\Cref{fig:spv:boxplot:reduced,tab:reduced}). In an additional simulation, we evaluated the double robustness of both the AIPW and TMLE SPV estimators (see \Cref{app:additional:figures,app:additional:setting}, \Cref{fig:dr,fig:dr:tree}).
 
\subsection{In-vitro fertilization application}
\label{subsec:ivf}
The IVF dataset consists of $n=18{,}538$ recorded ovarian stimulation cycles from multiple clinical centers. Observations comprise baseline characteristics, ordinal gonadotropin dosages ($\mathcal{A}=\{1,\ldots, 6\}$), and follicular yield ($Y$) to be maximized. The aim is to support clinical decision-makers in hormone dose selection, leaving room for clinical judgment. 

This real-world implementation illustrates how adjusting the parameter $r$ in conformal set-valued policy learning modulates set cardinality to reveal selection stability and exclusion order. We pre-specify a target confidence level $\alpha=0.1$ alongside a grid of randomness levels $r\in \{0, 0.3, 0.6\}$. We estimate the noisy labels with \texttt{MACF} and estimate the nonconformity score with a Super Learner ensemble \citep{van2007super, SuperLearner}. We report the policies for each $r$ as a sensitivity analysis. As shown in \Cref{fig:treatments:set}, decreasing $r$ reduces set cardinality by filtering high doses while consistently retaining median ones. This filtering pattern reveals a trend toward dose de-escalation that aligns with recent clinical evidence, offering a critical insight that single treatment policies fail to capture. Most patients are assigned multiple treatments, allowing personalized strategies such as using the lowest effective dose to reduce ovarian hyperstimulation syndrome (OHSS) risk. \Cref{fig:Y_Xi} illustrates how delegating the final choice allows clinicians to accommodate secondary objectives.

\begin{figure}[h]
    \centering
    \includegraphics[width=\linewidth, page=3]{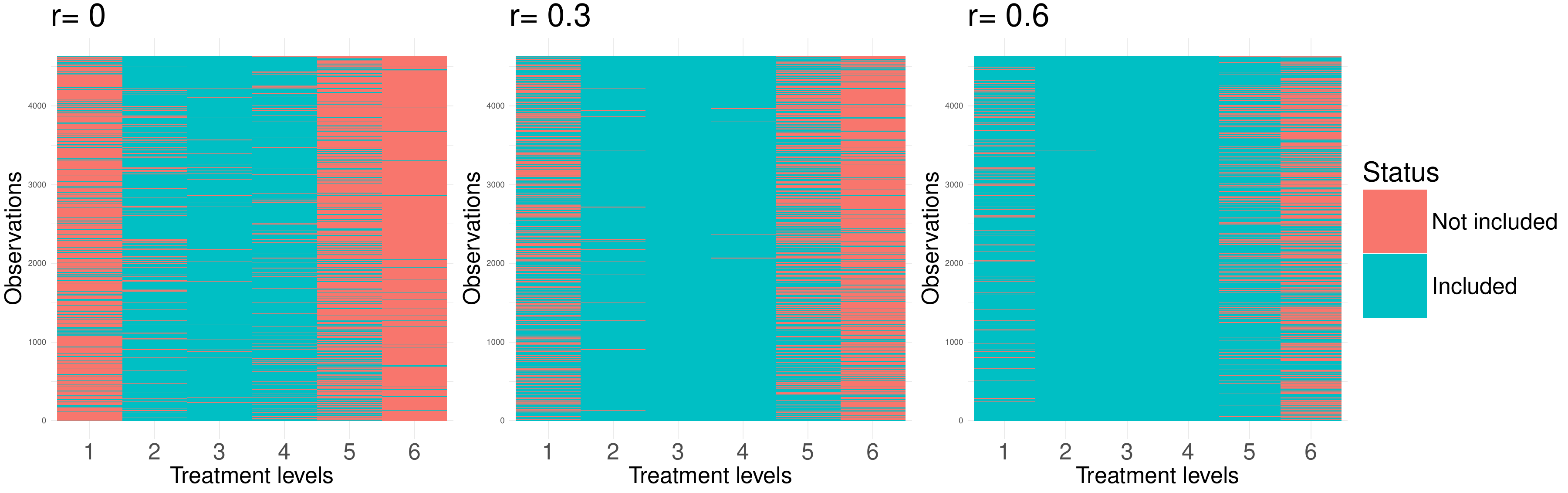}    
    \caption{Inclusion status under conformal policies across $r \in \{0, 0.3, 0.6\}$. Panels show observations (rows) by treatment level (columns), and colors indicate inclusion at $\alpha = 0.1$.}
    \label{fig:treatments:set}
    \vspace{-0.5cm}
\end{figure}

\subsection{Trauma care application}
\label{subsec:traumacare}
The trauma care training dataset consists of $n=32{,}480$ accident records. Each observation comprises patient and accident baseline characteristics ($X\in \mathbb{R}^{31}$), a massive transfusion indicator ($\mathcal{A}=\{0,1\}$), and the outcome ($Y\in \{0,1\}$), patient survival within 24h and 28 days. The goal is to assist clinical decision-makers in assigning transfusion to maximize survival. We construct prediction sets on an unlabeled test dataset ($n_{\mathrm{test}}=6{,}679$) that includes solely baseline covariates ($X$). Data-splitting can be used to identify the best candidate set-valued policy to be further presented to the decision-maker. 

We pre-define a grid of randomness levels $\tilde{r}= \{0, 0.1, 0.25, 0.5\}$, fix the target significance level at $\alpha=0.1$ and partition the training data into three disjoint subsets dedicated to policy learning, nuisance estimation (for SPVs), and evaluation. To construct set-valued policies, we fit GLB using both a probability forest and an interaction-augmented generalized linear model. In addition, we implement the conformal method for each $r\in \tilde{r}$ (\Cref{algo:CSP}) using an ensemble scheme to generate noisy labels and a probability forest to compute nonconformity scores. Nuisances for the AIPW and TMLE SPV estimators are learned using probability forests (refer to \Cref{app:estimation:details} for estimation details). Finally, model selection is performed on the evaluation subset. The AIPW and TMLE propensity SPV asymptotic confidence intervals (\Cref{fig:asymptotic:ci:traumacare}) show that all methods produce similar results, close to the clinician benchmark. We identify GLB GLM as the top candidate, which yields a 15\% test deferral rate.

We provide the dedicated code required to reproduce all results at \href{https://anonymous.4open.science/r/setValuedPolicyLearning-6B9C}{\texttt{setValuedPolicyLearning}}. 

\section{Conclusion}
\label{sec:conclusion}

Instead of systematically prescribing a single treatment, set-valued policy learning recommends a set of treatments, shifting the focus from prescribing decisions to supporting them. While this framework explicitly models uncertainty, it remains under-explored. To evaluate set-valued policies and address the variability they introduce in downstream decisions, we proposed the uniform and propensity SPV alongside AIPW and TMLE estimators. We introduced two set-valued policy
learning methods, namely GLB and conformal set-valued policy learning. In numerical experiments, both approaches achieve policy values comparable to or higher than those of baseline methods, while frequently recovering the full set of optimal treatments.  Although the conformal method cannot formally guarantee marginal coverage for a fixed $r$, coverage alone may be insufficient for evaluating set-valued policies as it does not penalize the inclusion of suboptimal treatments. The conformal approach's main strengths are its compatibility with any policy learning algorithm and the clinically informative insights obtained by varying $r$, as demonstrated in \Cref{subsec:ivf}. 
In contrast, GLB guarantees conditional coverage, albeit sometimes at the cost of producing larger sets.  Furthermore, GLB relies on valid upper- and lower-bound estimators, for which established methods remain non-trivial and scarce. 
\citet{luedtke2024one} propose confidence sets for complex parameters such as counterfactual density estimation, providing foundations to address such estimation challenges. In practice, we recommend following the procedure outlined in \Cref{subsec:traumacare} to identify a valuable set-valued policy that achieves large SPVs while maintaining a reasonable set cardinality.

\section*{Acknowledgments}
\label{acknowledgements}
We thank the Traumatrix group for sharing the trauma care dataset. This work was supported by the "IA pour la Santé" project, funded by ANITI and the Occitanie Region.

\bibliography{bibliography}
\bibliographystyle{plainnat}

\clearpage
\appendix
\thispagestyle{empty}

\onecolumn

\section{Details on set-policy values estimation}
\label{app:evaluation}

\subsection{Evaluation with fixed choice function}
\label{app:evaluation_fixed_choice_function}
In this section, we explicitly use the subscript $P$ to denote the dependence of the parameters on the data-generating distribution $P$. Elsewhere, we drop the subscript $P$ to avoid cluttering the notation.

Let $\mathcal{M}$ denote a statistical model, that is, a collection of laws for $\mathcal{D}$, such that for every $P\in \mathcal{M}$, the nuisances $(x,a) \mapsto \mu_{P}(x,a)=E[Y|X=x,A=a]$ and $(x,a) \mapsto e_{P}(a|x)=P(A=a|X=x)$ are well defined. We further define $m_{P}(X)=\sum_{a\in \mathcal{A}}q_{\delta}(X,a)\mu_{P}(X,a)$, where $q_{\delta}(X,a)= \Prob_{A'\sim \delta(C(X),X)}\left(A' = a \right)$. Let $\hat{m}$, $\hat{e}$ and $\hat{\mu}$ denote estimators of $m_{P}, e_{P}$ and $\mu_{P}$, respectively, trained on a sample independent of the evaluation data. 

\paragraph{Standard estimators for fixed choice functions.}

Suppose first that $\delta$ is known and does not depend on the observational distribution $P$. Under \Cref{ass:ci}, the SPV admits the outcome-regression and inverse-probability representations 
\begin{equation}
\mathcal V_{P}(C,\delta) = 
E_{X\sim P_{X}}[m_{P}(X)] = 
E_{(A,X)\sim e_{P}(A|X)P_{X}}\left[\frac{q_{\delta}(X,A)}{e_{P}(A|X)}Y\right],
\label{eq:spv-ident-fixed}
\end{equation}
where $E_{(A,X)\sim e_{P}(A|X)P_{X}}[\cdot]$ denotes the joint expectation characterized by $E_{X\sim P_{X}, A|X\sim e_{P}(A|X)}$. These suggest the G-formula and inverse-probability-weighted (IPW) estimators
\begin{equation}
\widehat{\mathcal V}_{\mathrm G}(C;\delta)
=\frac{1}{n}\sum_{i=1}^n\widehat m(X_{i}),
\qquad
\widehat{\mathcal V}_{\mathrm{IPW}}(C;\delta)
=\frac{1}{n}\sum_{i=1}^n
\frac{q_{\delta}(X_{i}, A_{i})}{\widehat e(A_{i}|X_{i})}Y_{i},
\label{eq:spv-basic-fixed}
\end{equation}
where $\widehat m(x)=\sum_{a\in \mathcal{A}} q_{\delta}(x,a)\widehat\mu(x,a)$. The G-formula relies on estimating the conditional mean, whereas IPW relies on estimating the propensity score. Combining both yields an augmented IPW (AIPW) estimator based on the efficient influence function (EIF) in the nonparametric observed-data model, that is on $\phi_{\mathrm{fixed}}$ given by
\begin{equation}
\phi_{\mathrm{fixed}}(\mathcal{D}) =\frac{q_{\delta}(X,A)}{e_{P}(A|X)}\{Y-\mu_{P}(X,A)\}
+m_{P}(X)-\mathcal V_P(C,\delta).
\label{eq:spv-eif-fixed}
\end{equation}
See proof in the next paragraph. We refer to \citet{hines2022demystifying} for a gentle introduction to EIF and their applications. 
The corresponding AIPW estimator (one-step or estimating-equation estimator) is:
\begin{equation}
\widehat{\mathcal V}_{\mathrm{AIPW}}(C;\delta)
=\frac{1}{n}\sum_{i=1}^n\left[
\widehat m(X_{i})
+\frac{q_{\delta}(X_{i},A_{i})}{\widehat e(A_{i}|X_{i})}
\{Y_{i}-\widehat\mu(X_{i},A_{i})\}\right].
\label{eq:spv-aipw-fixed}
\end{equation}

Assume $Y\in [y_{\min},y_{\max}]$ almost surely and let $Y^{01}=(Y-y_{\min})/((y_{\max}-y_{\min}))\in [0,1]$. Let $\mu_{P}^{01}(x,a) = E_{P}[Y^{01}|X=x,A=a]$ and let $\hat{\mu}^{01}$ denote an estimator of $\mu_{P}^{01}$ with values in $(0,1)$. To define the Targeted Maximum Likelihood Estimator (TMLE), we introduce a fluctuation of $\hat{\mu}^{01}$, specifically, the parametric submodel $\{\hat{\mu}^{01}(\epsilon):\epsilon\in \mathbb{R}\}$ characterized by: 
\begin{equation*}
    \hat{\mu}^{01}(\epsilon)(X,A) = \expit \left( \logit (\hat{\mu}^{01}(X,A)) + \epsilon \frac{q_{\delta}(X,A)}{\hat{e}(A|X)}\right),
\end{equation*} whose definition depends on the initial estimator $\hat{e}$ of $e_{P}$. Let 
\begin{equation}
\label{eq:fluctutation:param:fixed}
    \epsilon^{*} \in  \left\{ \argmin_{\epsilon \in \mathbb{R}} -\frac{1}{n}\sum_{i=1}^{n}\left[Y^{01}_{i}\log(\hat{\mu}^{01}(\epsilon)(X_{i},A_{i})) + (1-Y^{01}_{i})\log(1-\hat{\mu}^{01}(\epsilon)(X_{i}, A_{i}))\right] \right\}
\end{equation} and set $\hat{\mu}^{*} = (y_{\max}-y_{\min})\hat{\mu}^{01}(\epsilon^{*}) + y_{\min}$. The TMLE is the plug-in estimator
\begin{equation}
    \hat{\mathcal{V}}_{\mathrm{TMLE}}(C; \delta) = \frac{1}{n}\sum_{i=1}^{n}\hat{m}^{*}(X_{i}), 
\label{eq:spv-tmle-fixed}
\end{equation}where $\hat{m}^{*}(x) = \sum_{a\in \mathcal{A}}q_{\delta}(x,a)\hat{\mu}^{*}(x,a)$. 

Under suitable overlap and moment conditions, both AIPW and TMLE estimators are consistent if either the conditional mean or the propensity score are consistently estimated on the treatments used by the choice function, provided that  $\hat{m}$ and $\hat{g}_{C}$  belong to a Glivenko-Cantelli class. If both $\hat{m}$ and $\hat{g}_{C}$ are consistent estimators with a sufficiently small product of estimation errors, and if both belong to Donsker classes, then the AIPW and TMLE estimators are asymptotically normal and efficient. 

\paragraph{EIF derivation for fixed choice function.}
\begin{proof}
    Let $C\in \mathcal{P}(\mathcal{A})^{\mathcal{X}}$ be a fixed set-valued policy. We derive the efficient influence function for the target parameter $P\mapsto \mathcal{V}_{P}(C; \delta)$, with $\delta$ a fixed choice function. 

 Let $\{P_{t}: t\in (-1/M,M) \}$ be a parametric submodel through $P$, defined as follows. Let $s:\mathcal{D}\to \mathbb{R}$ be a measurable function with $\|s\|_{\infty}\leq M$ and $E_{P}[s(\mathcal{D})]=0$. For every $t\in (-1/M, M)$, let $P_{t}\ll P$ be the law with density $\tfrac{dP_{t}}{dP}=1+ts$. Then, $P_{t}|_{t=0}=P$ and the score of the submodel at $t=0$ is $s$: $\frac{d}{dt}\log(\frac{dP_{t}}{dP}(\mathcal{D}))|_{t=0} = s(\mathcal{D})$.

We say that $P \mapsto \mathcal{V}_{P}(C; \delta)$ is pathwise differentiable with respect to $\mathcal{M}$ at $P$ with influence function $\phi_{\mathrm{fixed}}(P)\in L_{2}^{0}(P)$ if
\begin{equation*}
    \frac{d}{dt} \mathcal{V}_{P_{t}}(C; \delta)\big|_{t=0} = \int \phi_{\mathrm{fixed}}(P)(\mathcal{D})s(\mathcal{D})dP(\mathcal{D}),
\end{equation*} for all $s$ as above. 

We denote $\mu_{t}(x,a) = E_{P_{t}}[Y|X=x,A=a]$, and $S=s(\mathcal{D})$. We assume that differentiation and expectation can be interchanged. Then, 

\begin{align*}
    \frac{d}{dt} \mathcal{V}_{P_{t}}(C; \delta)\big|_{t=0} &= \frac{d}{dt} E_{P_{t}}\left[\sum_{a\in \mathcal{A}}q_{\delta}(X,a)\mu_{t}(X,a)\right]\Bigg|_{t=0} = \frac{d}{dt} E_{P}\left[(1+tS)\sum_{a\in \mathcal{A}}q_{\delta}(X,a)\mu_{t}(X,a)\right]\Bigg|_{t=0} \\
    &=  E_{P}\left[S\sum_{a\in \mathcal{A}}q_{\delta}(X,a)\mu(X,a)\right] +  E_{P}\left[\sum_{a\in \mathcal{A}}q_{\delta}(X,a)\frac{d}{dt}\mu_{t}(X,a)\big|_{t=0}\right],
\end{align*}
where according to \citet[Lemma 1]{chambaz2012estimation}, $\frac{d}{dt} \mu_{t}(X,a)\big|_{t=0} = E_{P}[S(Y-\mu_{P}(X,a))|X,A=a]$, for every $a\in \mathcal{A}$. Combining these results yields
\begin{equation*}
    \phi_{\mathrm{fixed}}(P)(\mathcal{D}) = m_{P}(X) + \frac{q_{\delta}(X,A)}{e_{P}(A|X)}(Y-\mu_{P}(X,A)) -\mathcal{V}_{P}(C;\delta).
\end{equation*}

The above function $\phi_{\mathrm{fixed}}(P)$ is mean zero and  square integrable under $P$. It is therefore the EIF of $P\mapsto \mathcal{V}_{P}(C;\delta)$.

\end{proof}
\subsection{Distribution-dependent choice function}
\label{app:SPV_distrib_dependent_choice}

\paragraph{EIF derivation.}
We now focus on the distribution-dependent choice function case, $P \mapsto \delta_{P}$. Let $q_{P}(x,a)=\Prob_{A'\sim \delta_P(C(x),x)}(A' = a)$ be the conditional probability that the choice function assigns treatment $a\in \mathcal{A}$ given $X=x$. It thus holds that
\begin{equation*}
    q_{P}(x,a)= \bone\{a\in C(x)\}\frac{e_{P}(a|x)}{g_{C,P}(x)},
\end{equation*} where $g_{C,P}(x) = \sum_{a'\in C(x)}e_{P}(a'|x)$.

\begin{proof}[Proof of \Cref{prop:eif_dependent}]
Let $C\in \mathcal{P}(\mathcal{A})^{\mathcal{X}}$ be a fixed set-valued policy. We derive the efficient influence function for $P \mapsto \mathcal{V}_{P}(C;\delta_{P})$ where $\delta_P$ denotes the propensity choice function depending on $P$ in \Cref{def:propensity_SPV}. Recall the parametric submodel through $P$ introduced in the previous paragraph. 

To simplify notation, let $\delta_{t} = \delta_{P_{t}}$ and $q_{t}(x,a) = \bone\{a\in C(X)\}\frac{e_{t}(a\vert{}x)}{g_{C,t}(x)}$, with $e_{t}(a|x) = P_{t}(A=a|X=x)$ and $g_{C,t}(x) = \sum_{a'\in C(x)}e_{t}(a'|x)$. We assume  that differentiation and expectation can be interchanged. Differentiating the expected value with respect to $t$ yields
\begin{align*}
    \frac{d}{dt}\mathcal{V}_{P_{t}}(C; \delta_{t})\big|_{t=0} &= \frac{d}{dt}E_{P_{t}}\left[\sum_{a\in \mathcal{A}}q_{t}(X,a)\mu_{t}(X,a) \right]\Bigg|_{t=0} = \frac{d}{dt}E_{P}\left[(1+tS)\sum_{a\in \mathcal{A}}q_{t}(X,a)\mu_{t}(X,a) \right]\Bigg|_{t=0}\\
    &= E_{P}\left[S\sum_{a\in \mathcal{A}}q_{P}(X,a)\mu(X,a) \right] + 
    E_{P}\left[\sum_{a\in \mathcal{A}}\mu(X,a)\frac{d}{dt}q_{t}(X,a)\Big|_{t=0}\right] + E_{P}\left[\sum_{a\in \mathcal{A}}q_{P}(X,a)\frac{d}{dt}\mu_{t}(X,a)\Big|_{t=0}\right].
\end{align*}
To evaluate the term involving $\frac{d}{dt}q_{t}(x,a)\big|_{t=0}$, we use \citet[Lemma 1]{chambaz2012estimation} to establish the derivatives for $e_{t}(a|X)$ and $g_{C,t}(X)$:
\begin{align*}
   \frac{d}{dt} e_{t}(a|X)|_{t=0} &= \frac{d}{dt} \frac{e_{P}(a|X) + tE_{P}[S\bone\{A=a\}|X]}{1+tE_{P}[S|X]}\Big|_{t=0}\\
   &= E_{P}[S\bone\{A=a\}|X]-e_{P}(a|X)E_{P}[S|X],\\
   &= E_{P}[S(\bone\{A=a\}-e_{P}(a|X))|X],
\end{align*} 
\begin{equation*}
    \frac{d}{dt} g_{C,t}(X)|_{t=0} = E_{P}[S(\bone\{A\in C(X)\} - g_{C,P}(X))|X]
\end{equation*}
hence
\begin{equation*}
    \frac{d}{dt}q_{t}(X,a)\big|_{t=0} = \frac{\bone_{a\in C(X)}}{g_{C,P}(X)}\left(\frac{d}{dt}e_{t}(a|X)|_{t=0} - q_{P}(X,a)\frac{d}{dt}g_{C,t}(X)|_{t=0}\right).
\end{equation*}

Finally, 
\begin{align*}
    \sum_{a\in \mathcal{A}}\mu_{P}(X,a)\frac{d}{dt}q_{t}(X,a)\big|_{t=0} &= \frac{1}{g_{C,P}(X)}\left(\sum_{a\in C(X)}\mu_{P}(X,a)\frac{d}{dt}e_{t}(X,a)\big|_{t=0} - m_{P}(X)\frac{d}{dt}g_{C,t}\Big|_{t=0}\right)\\
    &= \frac{1}{g_{C,P}(X)} E_{P}[S(\bone\{A\in C(X)\}\mu_{P}(X,A) - g_{C,P}(X)m_{P}(X)\\
    &\quad \quad \quad - m_{P}(X)(\bone\{A\in C(X)\} - g_{C,P}(X)))|X]\\
    &= E_{P}\left[S\frac{\bone\{A\in C(X)\}}{g_{C,P}(X)} (\mu_{P}(X,A) - m_{P}(X))|X\right]
\end{align*}

Finally, collecting the above results yields the efficient influence function, given by
\begin{align*}
    \phi(P)(\mathcal{D}) &= m_{P}(X) +  \frac{q_{P}(X,A)}{e_{P}(A|X)}(Y-\mu_{P}(X,A)) + \frac{\bone\{A\in C(X)\}}{g_{C,P}(X)}(\mu_{P}(X,A) - m_{P}(X)) - \mathcal{V}_{P}(C; \delta_{P})\\
    &=\phi_{\mathrm{fixed}}(P)(\mathcal{D}) + \frac{\bone\{A\in C(X)\}}{g_{C,P}(X)}(\mu_{P}(X,A) - m_{P}(X)).
\end{align*} Note that the expression can be further simplified to $\phi(P)(\mathcal{D}) = m_{P}(X) + \frac{\bone\{A\in C(X)\}}{g_{C,P}(X)}(Y-m_{P}(X)) - \mathcal{V}_{P}(C; \delta_{P})$. 
\end{proof}

\paragraph{Double-robustness for AIPW.}
We now prove \Cref{thm:double_robustness}, the double robustness property of the AIPW estimator of the propensity set-policy value: 
\begin{equation*}
    \hat{\mathcal{V}}^{\mathrm{AIPW}}_{e}(C) = \frac{1}{n}\sum_{i=1}^{n}\hat{m}(X_{i}) +\frac{\bone\{A_{i}\in C(X_{i})\}}{\hat{g}_{C}(X_{i})}(Y_{i}-\hat{m}(X_{i})), 
\end{equation*} where $\hat{g}_{C}(x) = \sum_{a\in C(x)}\hat{e}(a|x)$ is an estimator of $g_{C,P}(x)$. Note that the AIPW estimator is both the one-step and the estimating-equation estimators. 

\begin{proof}[Proof of \Cref{thm:double_robustness}]
\label{proof:AIPW}
    Assume that $g_{C,P}\geq\eta$, $\hat g_{C}\geq\eta$ for some $\eta>0$, that $\hat{m},\hat{g}_{C}$ are learned on independent data with $\|\hat m-m_{P}\|_2\to0$, $\|\hat g_{C}-g_{C,P}\|_2\to0$ and \Cref{ass:ci} holds.
    Let $\tau(\mathcal{D};\hat g_{C},\hat m) = \hat m(X) + \frac{\bone\{A\in C(X)\}}{\hat g_{C}(X)}(Y-\hat m(X))$. It holds that
    \begin{equation}
    \label{eq:AIPW:decomp}
         \sqrt{n}\left(\hat{\mathcal{V}}_e^\mathrm{AIPW}(C) - \mathcal{V}_e(C)\right) = \frac{1}{\sqrt{n}}\sum_{i=1}^n \Big(\tau(\mathcal{D}_{i};\hat g_{C},\hat m) - E_{P}[\tau(\mathcal{D};\hat g_{C},\hat m)]\Big)
        + \sqrt{n}\,E_{P}\left[ \tau(\mathcal{D};\hat g_{C},\hat m) - m_{P}(X)\right]\,.
    \end{equation}
   \emph{First term.} Conditionally on the data used to learn $\hat m,\hat g_{C}$:
   $\tau(\mathcal{D}_{i};\hat g_{C},\hat m)$ are i.i.d.\ and bounded, so the CLT applies with variance $\mathrm{Var}_{P}(\tau(\mathcal{D};\hat g_{C},\hat m))$. Let $f(\mathcal{D})=m_{P}(X)+\frac{\bone\{A\in C(X)\}}{g_{C,P}(X)}(Y-m_{P}(X))$. Since $\hat g_{C}\geq\eta$, $g_{C,P}\geq\eta$ and $\hat m\to m_P$, $\hat g_C\to g_{C,P}$ in $L_2(P)$, we have $\|\tau- f\|_{2}\to0$, hence
    $\mathrm{Var}_{P}(\tau(\mathcal{D};\hat g_{C},\hat m))\to \sigma_e^2(C)=\mathrm{Var}_{P}(m_{P}(X)) + E_{P}\left[\frac{\mathrm{Var}_{P}(Y|X,A\in C(X))}{g_{C,P}(X)}\right]$. Therefore, the first term converges in distribution to $\mathcal{N}(0,\sigma_e^2(C))$ conditionally on the data used to learn $\hat{m}$ and $\hat{g}_{C}$, and hence marginally. 

\textit{Second term.} Since $E_{P}[\bone\{A\in C(X)\}Y\mid X]=g_{C,P}(X)m_{P}(X)$,
    \begin{align*}
        E_{P}\left[\tau(\mathcal{D};\hat g_{C},\hat m)\mid X\right] - m_{P}(X)
        &= \hat m(X)-m_{P}(X) + \frac{g_{C,P}(X)}{\hat g_{C}(X)}(m_{P}(X)-\hat m(X))\\
        &= \frac{\hat g_{C}(X)-g_{C,P}(X)}{\hat g_{C}(X)}(\hat m(X)-m_{P}(X))\,.
    \end{align*}
    Taking expectations and using $\hat g_{C}\geq\eta$ and the Cauchy-Schwarz inequality yields,
    \begin{equation*}
        \sqrt{n}\left|E_{P}\left[\tau(\mathcal{D};\hat g_{C},\hat m) - m_{P}(X)\right]\right|
        \leq \sqrt{n}\,\eta^{-1}\|\hat g_{C}-g_{C,P}\|_2\,\|\hat m-m_{P}\|_2\,.
    \end{equation*}
 Plugging both bounds into \eqref{eq:AIPW:decomp} gives the claim.
\end{proof}

\paragraph{Double-robustness for TMLE.}
To define the TMLE estimator, recall the notations introduced in \Cref{app:evaluation_fixed_choice_function}. The new parametric submodel $\{\hat{\mu}^{01}(\epsilon):\epsilon \in \mathbb{R}\}$ depends on the estimator $\hat{g}_{C}$ of $g_{C,P}$ as follows: 
\begin{equation}
\label{eq:loss:tmle}
    \hat{\mu}^{01}(\epsilon)(X,A) = \expit \left( \logit(\hat{\mu}^{01}(X,A)) + \epsilon \frac{\bone\{A\in C(X)\}}{\hat{g}_{C}(X)} \right). 
\end{equation} The optimal fluctuation parameter $\epsilon^{*}$ minimizes the same loss as in \Cref{eq:fluctutation:param:fixed}, so that
\begin{equation}
\label{eq:fluct:param:P}
    \frac{1}{n}\sum_{i=1}^{n}\frac{\bone\{A_{i}\in C(X_{i})\}}{\hat{g}_{C}(X_{i})}(Y_{i}^{01}-\hat{\mu}^{01}(\epsilon^{*})(X_{i},A_{i}))=0.
\end{equation} Define $\hat{\mu}^{*}(x,a) = (y_{\max}-y_{\min})\hat{\mu}^{01}(\epsilon^{*})+y_{\min}$ and $\hat{m}^{*}(X)=\sum_{a\in \mathcal{A}}q_{P}(x,a)\hat{\mu}^{*}(x,a)$.

\begin{theorem}[Double Robustness for TMLE]\label{thm:double_robustness-tmle}
    Let $\hat m,\hat g_C$ be estimators of $m$ and $g_C$ learned on independent data. 
   Assume that $\inf_{x\in\mathcal{X}}\hat{g}_{C}(x)\geq \eta$ for some $\eta>0$, that $Y\in [y_{\min},y_{\max}]$ and that \Cref{ass:ci} holds. The TMLE estimator writes as
    \begin{equation*}
       \hat{\mathcal{V}}_e^\mathrm{TMLE}(C) = \frac{1}{n}\sum_{i=1}^n \hat{m}^{*}(X_{i}).
    \end{equation*}
    It satisfies:
    \begin{align*}\label{eq:double_robustness}
        \sqrt{n}\big(\hat{\mathcal{V}}_e^\mathrm{TMLE}(C) &- \mathcal{V}_e(C)\big) + O_{P}\left(\sqrt{n}\eta^{-1}\|\hat{m}^{*}-m_{P}\|_2\|\hat{g}_C-g_{C,P}\|_2\right)\, \leadsto \mathcal{N}\left(0,\sigma_e^2(C)\right),
    \end{align*}
    where 
    \begin{equation*}
       \sigma_e^2(C) = \mathrm{Var}_{P}(\hat{m}^{*}(X)) + E_{P}\left[\frac{\mathrm{Var}_{P}\left(Y|X,A\in C(X)\right)}{g_{C,P}(X)}\right]\,.
    \end{equation*}
\end{theorem}

By construction we have \Cref{eq:fluct:param:P}. It holds that $\hat{\mathcal{V}}_{e}^{\mathrm{TMLE}}(C)=\frac{1}{n}\sum_{i=1}^{n}\tau^{*}(D_{i})$, with 
\begin{equation*}
    \tau^{*}(\mathcal{D})=\hat m^{*}(X)+\frac{\bone\{A\in C(X)\}}{\hat g_{C}(X)}\big(Y-\hat\mu^{*}(X,A)\big).
\end{equation*}
\begin{proof}
  The proof follows analogously to that of \Cref{thm:double_robustness}, replacing $\tau(\mathcal{D};\hat{g}_{C},\hat{m})$ with $\tau(\mathcal{D};\hat{g}_{C},\hat{m}^{*})$.
\end{proof}

\newpage
\section{Additional technical details}
\label{app:sec:proofs}

\subsection{Uniform set-policy value (\Cref{sec:problem:setup})}
\label{app:proof:lower:bound:SPV}
Let $\delta^{*}$ denote the decision-maker's choice strategy. 

\begin{assumption}[Decision-maker outperforms random]
\label{ass:decision:maker:strategy}
    For all $x \in \mathcal{X}$ and $C \in \mathcal{P}(\mathcal{A})^{\mathcal{X}}$,
    \begin{equation*}
        \frac{1}{|C(x)|} \sum_{a \in C(x)} \mu(x, a) \leq \mathbb{E}_{A' \sim \delta^{*}(C(x), x)} [\mu(x, A')]\,.
    \end{equation*}
\end{assumption}

\begin{proposition}
\label{prop:lower:bound:SPV}
    If the decision-maker's choice function $\delta^{*}$ outperforms uniform sampling under \Cref{ass:decision:maker:strategy}, then $\mathcal{V}_{P}(C; \delta_{\mathrm{unif}}) \leq \mathcal{V}_{P}(C; \delta^{*})$.
\end{proposition}

\begin{proof}
Let $X\in \mathcal{X}$ and $C\in \mathcal{P}(\mathcal{A})^{\mathcal{X}}$ be fixed. Under \Cref{ass:decision:maker:strategy}, the decision-maker's choice $\delta^{*}$ outperforms the random strategy. Taking the expectation with respect to $P_{X}$ yields the desired inequality:
\begin{align*}
     & \mathcal{V}_P(C; \delta_{\mathrm{unif}}) =E_{X\sim P_{X}}\left[\frac{1}{|C(X)|}\sum_{a\in C(X)}\mu(X,a)\right] \leq E_{(A',X)\sim\delta^{*}_{C} P_{X}}[\mu(X,A')]=\mathcal{V}_{P}(C;\delta^{*}).
\end{align*}
\end{proof}

\subsection{Proof of \Cref{prop:naive}}
\label{app:proof:naive}
\begin{proof} Fix $x\in \mathcal{X}$, $a^{*}\in \Pi^{\star}(x)$, and note that 
\begin{gather}
\label{eq:obvious}
    \mu(x,\underline{a}(x))\leq \mu(x,a^{*}),\\
\label{eq:remark:naive}
        a^{*} \in \hat{C}_{\alpha}(x) \quad \text{if and only if} \quad \hat{u}_{\alpha}\left(x,a^{*}\right)\geq \hat{\ell}_{\alpha}\left(x,\underline{a}(x)\right).
\end{gather}
Introduce the events $E_{1} = \{\hat{\ell}_{\alpha}\left(x,\underline{a}(x)\right) > \mu(x,\underline{a}(x))\}$ and 
$E_{2} = \{\hat{u}_{\alpha}\left(x,a^{*}\right) < \mu(x,a^{*})\}$. By construction, each event has probability $\alpha/2$, hence their union has probability at most $\alpha$.

In view of \Cref{eq:obvious} (first equality) and \Cref{eq:remark:naive} (second equality), it holds that 
\begin{align*}
   (E_{1}\cup E_{2})^{c}&=\{\hat{\ell}_{\alpha}(x,\underline{a}(x)) \leq \mu(x,\underline{a}(x)) \leq \mu(x,a^{*}) \leq  \hat{u}_{\alpha}(x,a^{*}\}\\
    &\subseteq  \{\hat{\ell}_{\alpha}(x,\underline{a}(x)) \leq  \hat{u}_{\alpha}(x,a^{*})\} \\
    &= \{a^{*} \in \hat{C}_{\alpha}(x)\} \subseteq \{\Pi^{\star}(x)\cap \hat{C}_{\alpha}(x)\neq \emptyset\}.
\end{align*}
Consequently, 
\begin{equation*}
    \Prob(\Pi^{\star}(x)\cap \hat{C}_{\alpha}(x)\neq \emptyset) \geq 1 - \Prob(E_{1}\cup E_{2}) \geq 1- \alpha,
\end{equation*}
thus establishing the desired marginal coverage property across the entire covariate space $\mathcal{X}$. The statement in \Cref{prop:naive} follows straightforwardly.
\end{proof}

\subsection{Proof of \Cref{prop:set_cardinality}}
\label{proof:prop:cardinality}

We work conditionally on the nonconformity score function $\hat{s}$. By definition, for all $t \in \mathbb{R}$, we have that $F_{\mathrm{rd}}(t) = \Prob\left( \hat{s}(X_{i}, A_{i}^{\mathrm{rd}}) \leq t \right)$, $\hat{F}(t)= \Prob\left( \hat{s}(X_{i}, \hat{A}^{*}_{i}) \leq t \right)$, and (by independence of $R_{i}$ with respect to $(A_{i}^{\mathrm{rd}}, \hat{A}^{*}_{i})$)
\begin{equation}
\label{eq:CDF:perturbed}
    \hat{F}_{r}(t) = r F_{\mathrm{rd}}(t) + (1-r)\hat{F}(t).
\end{equation}


\begin{proof}
Fix $r_{1}, r_{2} \in [0,1]$ with $r_{1} \leq r_{2}$, and 
let $\hat{F}_{r_1}$ and $\hat{F}_{r_2}$ be the corresponding CDFs (\Cref{eq:CDF:perturbed}). Since $F_{\mathrm{rd}}(\cdot) \leq \hat{F}(\cdot)$ by assumption, it holds that
\begin{align*}
    \hat{F}_{r_1}(\cdot) - \hat{F}_{r_2}(\cdot) = (r_{1} - r_{2})\cdot (F_{\mathrm{rd}}(\cdot) - \hat{F}(\cdot)) \geq 0.
\end{align*}
Because quantile functions reverse the point-wise ordering of CDFs, taking the generalized inverse yields $\hat{q}^{r_{1}}_{1-\alpha} \leq \hat{q}^{r_{2}}_{1-\alpha}$ where $\hat{q}^{r_{1}}_{1-\alpha}$ and $\hat{q}^{r_{2}}_{1-\alpha}$ are the $(1-\alpha)$-quantiles of $\hat{F}_{r_1}$ and $\hat{F}_{r_2}$ respectively. Therefore, $\hat C^{r_1}_\alpha (X_{n+1}) \subseteq \hat C^{r_2}_\alpha (X_{n+1})$. This completes the proof.
\end{proof}

\subsection{Proof of \Cref{thm:coverage:r}}
\label{app:proof:thm}

Consider the noisy calibration set $\{\hat{\mathcal{O}}_{i}=(X_{i}, \hat{A}^{*}_{r,i})\}_{i\in \mathcal{I}_{\mathrm{cal}}}$ where the labels $\hat{A}^{*}_{r,i}$ (\Cref{eq:noisy:label:generation}) combine the purely random labels $A_{i}^{\mathrm{rd}}$ and the estimated labels $\hat{A}^{*}_{i} = \mathcal{B}(\mathcal{D}_{b})(X_{i})$ at  level $r\in [0,1]$.    
Let $t \mapsto F(t) = \Prob(s(X_{i},A_{i}^{*})\leq t)$ be the CDF under the true labels. By assumption, $F_{\mathrm{rd}}(\cdot) \le \hat{F}(\cdot)$ and $F_{\mathrm{rd}}(\cdot)\leq F(\cdot)$.  Recall the \textit{inflation factor} from \citet{sesia25adaptative}, $\Delta_{r}(\cdot) = F(\cdot) - \hat{F}_{r}(\cdot)$.

\begin{proof}
    By \citet[Theorem A1]{sesia25adaptative}, a confidence set $\hat{C}^{r}_{\alpha}$ constructed via \Cref{algo:CSP} using the perturbed labels satisfies:
\begin{align*}
    \Prob(\Pi^{\star}(X_{n+1}) \cap \hat{C}^{r}_{\alpha}(X_{n+1})\neq \emptyset)
    &\geq 1-\alpha + E[\Delta_{r}(\hat{q}^{r}_{1-\alpha})],
\end{align*} where the expectation is taken with respect to the randomness of $\hat{q}^{r}_{1-\alpha}$. Therefore, marginal coverage (\Cref{eq:marginal:cov:A}) is satisfied
whenever $E[\Delta_{r}(\hat{q}^{r}_{1-\alpha})] \geq 0$. 

Under the assumed continuity of $\hat{F}_{r} (\cdot)$ at $\hat{q}_{1-\alpha}^{r}$, the condition $E[\Delta_{r}(\hat{q}^{r}_{1-\alpha})] \geq 0$ rewrites as $E[F(\hat{q}^{r}_{1-\alpha})] \geq 1 - \alpha$, since $\hat{F}_{r} (\hat{q}^{r}_{1-\alpha}) = 1 - \alpha$.  This alternative formulation motivates the introduction of  $\mathcal{R} = \left\{r\in [0,1]: E[F(\hat{q}_{1-\alpha}^{r})] \geq  1-\alpha\right\}$ and of its infimum $\bar{r} = \inf \mathcal{R}$. 

Under the assumptions of the theorem, $1 \in \mathcal{R}$, proving that the set is non-empty. Indeed, $\hat{F}_{r=1} (\cdot) = F_{\mathrm{rd}} (\cdot) \leq F(\cdot)$ hence, by continuity of $\hat{F}_{r=1} (\cdot)$ at $\hat{q}_{1-\alpha}^{r=1}$, $F(\hat{q}_{1-\alpha}^{r=1}) \geq \hat{F}_{r=1} (\hat{q}_{1-\alpha}^{r=1}) = 1 - \alpha$, which implies in turn $E[F(\hat{q}_{1-\alpha}^{r=1})] \geq 1 - \alpha$. Consequently, $\bar{r}$ is a well-defined element of $[0,1]$.

We have justified in the proof of \Cref{prop:set_cardinality} that, for any $0 \leq r \leq r' \leq 1$, $\hat{q}^{r}_{1-\alpha} \leq \hat{q}^{r'}_{1-\alpha}$, hence $F(\hat{q}^{r}_{1-\alpha}) \leq F(\hat{q}^{r'}_{1-\alpha})$ and $E[F(\hat{q}^{r}_{1-\alpha})] \leq E[F(\hat{q}^{r'}_{1-\alpha})]$. This implies that either $\mathcal{R} = (\bar{r}, 1]$ or $\mathcal{R} = [\bar{r}, 1]$. Consequently, 
since $\hat{C}^{\bar{r}}_{\alpha}(x) \subseteq \hat{C}^{r}_{\alpha}(x)$ for any $x\in \mathcal{X}$ and $r \in (\bar{r},1]\cup \{1\}$ (see \Cref{prop:set_cardinality}), marginal coverage holds as stated in the theorem. 
\end{proof}

\subsection{Rank-preserving noisy labels (\Cref{subsec:Conformal policy learning})}
\label{app:rank:preserving}
By definition, for each $i\in \mathcal{I}_{\mathrm{cal}}$,
$\mu(X_{i}, A_{i}^{*}) = \max_{a \in \mathcal{A}} \mu(X_{i},a)$, hence  
$s^{*}(X_{i}, A^{*}_{i}) \leq 0$ and, further, for all $a\in \mathcal{A}$,
\begin{equation}
\label{eq:s0:inequality}
\textstyle
    s^{*}(X_{i}, A^{*}_{i})\leq s^{*}(X_{i},a).
\end{equation} 
This implies FOSD $\hat{F} \succeq_{(1)} F$ (see \Cref{def:SOC}) 
and straightforwardly ensures marginal coverage (\Cref{eq:marginal:cov:A}) \cite[Theorem 1]{alaa2023confmetalrpredinf}. In practice, we rather use the empirical nonconformity score (\Cref{eq:non:conformity:score:reg}) based on an estimator 
$\hat{\mu}$ of $\mu$, which may violate \Cref{eq:s0:inequality} and thus compromise coverage guarantees. In such cases, nonconformity scores can assign lower values to suboptimal treatments than to optimal ones, thereby favoring similarity to incorrect labels.
However, the following assumption, which ensures that empirical nonconformity scores used in practice preserve treatment orders, is sufficient to ensure coverage.
\begin{assumption}[Rank preserving]
\label{ass:equal:ranking}
For all $a,a'\in \mathcal{A}$, $\mu(\cdot,a')\leq \mu(\cdot,a)$ implies $\hat{\mu}(\cdot, a')\leq \hat{\mu}(\cdot,a)$. 
\end{assumption} 
\begin{proposition}
\label{prop:ranking:ncs}
    Under \Cref{ass:equal:ranking}, 
    $C=\hat{C}_{\alpha}$  defined in \Cref{eq:noisy_policy_set} satisfies \Cref{eq:marginal:cov:A}.
\end{proposition} 

\begin{proof}
    Let $a \mapsto \hat{\mu}(\cdot,a)$ be an estimator of $a \mapsto \mu(\cdot,a)$ satisfying \Cref{ass:equal:ranking} and introduce $\hat{s}$, the empirical nonconformity based on $\hat{\mu}$ (\Cref{eq:non:conformity:score:reg}). For every fixed $x\in \mathcal{X}$, and every $a^{*}\in \Pi^{\star}(x)$, the inequality
    \begin{equation*}
        \hat{\mu}(x,a^{*})\geq \hat{\mu}(x,a),
    \end{equation*} holds for all $a\in \mathcal{A}$ under the ranking assumption. It follows, that, 
    \begin{equation*}
    \hat{s}(x,a^{*}) - \hat{s}(x,a) = \max_{a'\neq a^{*}}\hat{\mu}(x,a') - \hat{\mu}(x,a^{*}) -(\max_{a'\neq a}\hat{\mu}(x,a') - \hat{\mu}(x,a))\leq 0\\
    \end{equation*}
    Then, for all $t$, $\Prob(\hat{s}(x,a^{*})\leq t) \geq \Prob(\hat{s}(x,a)\leq t)$, implying $\hat{F} \succeq_{(1)} F$.  This guarantees marginal coverage (\Cref{eq:marginal:cov:A}) for any $\alpha>0$ \citep[Theorem 1]{alaa2023confmetalrpredinf}.  
\end{proof}

\newpage
\section{Details on conformal prediction}
\subsection{Conformal prediction (CP)}
\label{app:cp}
Conformal prediction (CP) \citep{saunders1999transduction,tutorialcp} is a model-agnostic and distribution-free framework for uncertainty quantification, providing prediction sets with guaranteed marginal coverage properties. Formally, CP operates on a dataset of $n$ i.i.d.\ samples $\{\mathcal{O}_{i} =(X_{i},T_{i}) \}_{i\in [n]}$ and a test point $\mathcal{O}_{n+1}=(X_{n+1},T_{n+1})$, all sampled from an unknown distribution $P_{X}P_{T|X}$. Each data structure $\mathcal{O}_{i}$ comprises a covariate vector $X_{i}\in \mathcal{X}$ and a categorical label $T_{i}\in \mathcal{T}=[K]$.  

\textbf{Inductive CP (ICP) \citep{inductiveconfpapadopoulos}} partitions the data into disjoint training and calibration subsets. The training data is used to train a nonconformity score function $s:\mathcal{X}\times \mathcal{T} \rightarrow \mathbb{R}$, which quantifies how unusual a candidate pair $(x,t)$ appears relative to the model's predictions (by convention, lower scores indicate higher conformity). The calibration subset is then used to compute an empirical quantile $q_{1-\alpha}$ over the calibration nonconformity scores, establishing a threshold for set inclusion. A prediction set $C_{\alpha}(X_{n+1})$ is then constructed by gathering all labels whose nonconformity score does not exceed $q_{1-\alpha}$. 
\begin{algorithm}[H]
\caption{Inductive Conformal Prediction (ICP)}
\label{algo:ICP}
\begin{algorithmic}
\Require Data $\{\mathcal{O}_{i}=(X_{i},T_{i})\}_{i\in [n+1]}$, \\ level $\alpha \in [0,1]$. 
\State Partition data into two subsets indexed by  $\mathcal{I}_{\mathrm{train}}$ and $\mathcal{I}_{\mathrm{cal}}$. Let $n_{\mathrm{cal}}$ be the cardinality of  $\mathcal{I}_{\mathrm{cal}}$.
\State Learn nonconformity score function $s$ using $\{\mathcal{O}_{i}:i\in \mathcal{I}_{\mathrm{train}}\}$.
\State Evaluate the nonconformity score on the calibration subset:
\For{$i \in \mathcal{I}_{\mathrm{cal}}$}
    \State $S_{i} = s(X_{i}, T_{i})$,
\EndFor
\State Define $q_{1-\alpha} = \mathrm{Quantile}\left(\tfrac{\lceil(1-\alpha)(1+n_{\mathrm{cal}})\rceil}{n_{\mathrm{cal}}}; \{S_{i}\}_{i \in \mathcal{I}_{\mathrm{cal}}}\right)$.
\State \Return $C_{\alpha}(X_{n+1})=\{t\in \mathcal{T}: s(X_{n+1},t) \leq q_{1-\alpha}\}\subseteq \mathcal{P}(\mathcal{T})$.
\end{algorithmic}
\end{algorithm}
Under exchangeability of the data, this construction guarantees finite-sample marginal coverage, namely
\begin{equation}
\label{eq:marginal:cov}
    \Prob(T_{n+1}\in C_{\alpha}(X_{n+1})) \geq 1-\alpha,
\end{equation}
without imposing distributional assumptions on the data-generating process. 

\subsection{CP under label noise.}
\label{app:noisy_ICP}
 In many cases high quality labels reflecting ground truth are unavailable, and one needs to rely instead on noisy surrogates $\hat{T}_{i}\in \mathcal{T}$ such as user-supplied annotations or estimated labels \citep{cauchoisweaksupervision, cheng2022many} as proxies for the true labels $T_{i}\in \mathcal{T}$. In this setting, known as the \textit{noisy label regime} \cite{einbinderlabelnoise}, one can employ inductive conformal prediction to construct prediction sets on top of the noisy observations $\{\hat{O}_{i}=(X_{i},\hat{T}_{i})\}_{i\in [n+1]}\sim P_{X}P_{\hat{T}|X}$. However, the presence of noise fundamentally alters the ICP procedure \Cref{algo:ICP} as it yields a modified nonconformity score distribution $\{\hat{S}_{i}=s(X_{i},\hat{T}_{i})\}_{i\in\mathcal{I}_{\mathrm{cal}}}$. Consequently, the empirical quantile $\hat{q}_{1-\alpha} = \mathrm{Quantile}(\frac{\lceil(1-\alpha)(1+n_{\mathrm{cal}})\rceil}{n_{\mathrm{cal}}};\{\hat{S}_{i}\}_{i\in \mathcal{I}_{\mathrm{cal}}})$ differs from $q_{1-\alpha}$, leading to a altered prediction set 
\begin{equation}\label{eq:ICP_set}
    \hat{C}_{\alpha}(X_{n+1})=\{t\in \mathcal{T}: s(X_{n+1},t) \leq  \hat{q}_{1-\alpha}\},
\end{equation} ensuring marginal coverage for noisy labels under exchangeability assumptions. However, this property does not extend to the true labels, 
\begin{equation}
\label{eq:marginal:cov:objective}
    \Prob(T_{n+1}\in \hat{C}_{\alpha}(X_{n+1})) \geq 1-\alpha, 
\end{equation} without additional assumptions.  Recent works by \citet{alaa2023confmetalrpredinf,sesia25adaptative} address this coverage gap 
under certain assumptions on the noisy conformity scores. To characterize when coverage guarantees transfer from noisy to true labels, let $F(t)=\Prob(s(X_{i},T_{i})\leq t)$ and $\hat{F}(t)= \Prob(s(X_{i},\hat{T}_{i})\leq t)$ denote the cumulative distribution functions (CDFs) of the nonconformity scores evaluated at the true and noisy labels. 

\begin{definition}[Stochastic dominance \cite{alaa2023confmetalrpredinf}]
\label{def:SOC}
    Let $G$ and $H$ be two CDFs. $G$ has \textit{first-order stochastic dominance} (FOSD) on $H$, $G\succeq_{(1)} H$, iff $G\leq H$; $G$ has \textit{second-order stochastic dominance} (SOSD) over $H$, $G \succeq_{(2)} H$, iff $\int_{-\infty}^{t}[H(u) - G (u)] du \geq 0$ for all $t \in \mathbb{R}$.
\end{definition} 

\citet[Theorem 1]{alaa2023confmetalrpredinf} shows that desired coverage \Cref{eq:marginal:cov:objective} holds for confidence levels $\alpha\in [0,\alpha^{*}]$ provided that at least one of the following stochastic ordering conditions hold,
\begin{equation*}
(i) \ \hat{F} \succeq_{(1)} F, \quad (ii) \ F \succeq_{(2)} \hat{F}, \quad \text{or} \quad (iii) \ \hat{F} \succeq_{\mathrm{mcx}} F.
\end{equation*} where $\succeq_{(1)}$ denotes first-order stochastic dominance (FOSD), $\succeq_{(2)}$ denotes second-order stochastic dominance, and $\succeq_{\mathrm{mcx}}$ the monotone convex order (\citet[Section A.1]{alaa2023confmetalrpredinf}). The level $1-\alpha^{*}$ is defined as the lowest level $t$, such that the inequality $F(t)\geq \hat{F}(t)$ holds, thereby defining the region $[0, \alpha^{*}]$, which is difficult to characterize in practice. Note that $(i)$ implies $\alpha^{*}=1$. 

In \citet[Theorem A1]{sesia25adaptative}, the authors impose a weaker condition requiring the \textit{inflation factor} $t \mapsto 
\Delta(t) =F(t)-\hat{F}(t)$ to be non-negative at $t=\hat{q}_{1-\alpha}$ to guarantee \Cref{eq:marginal:cov:objective}. For a confidence set $\hat{C}_{\alpha}$ built using noisy labels, they prove that, 
\begin{equation}
\label{eq:sesia:ext}
    \Prob(T_{n+1}\in \hat{C}_{\alpha}(X_{n+1}))\geq 1-\alpha +E[\Delta(\hat{q}_{1-\alpha})].
\end{equation}
A sufficient condition is to ensure $\alpha \in [0, \alpha^{*}])$.

\subsection{Illustrations for randomness injection under label noise (\Cref{subsec:Conformal policy learning})}
\label{app:intuition}

\definecolor{cbBlue}{rgb}{0.337, 0.706, 0.914}
\definecolor{cbPink}{rgb}{0.8, 0.475, 0.655}
\definecolor{cbGreen}{rgb}{0, 0.62, 0.451}

\textbf{Nonconformity quantiles and randomness injection regimes. }

\begin{figure}[H]
\centering
\begin{tikzpicture}[scale=2]
  \draw[thick, ->] (-1, 0) -- (3, 0) node[right] {$x$};
  \draw[thick, ->] (0, 0) -- (0, 1.2) node[above]{}; 
  
 \draw[cbBlue, thick, smooth] plot[domain=0:3, samples=100](\x, {1/(1 + exp(-5*(\x - 1.5)))});
 \draw[cbPink, thick, smooth] plot[domain=-1:3, samples=100] 
    (\x, {1/(1 + exp(-5*(\x-0.5)))});
  
  \draw[dashed, gray] (2.5, 1) -- (0, 1);
  \draw[cbBlue, dashed] (1.71, 0) -- (1.71, 0.75);
  \draw[dashed, cbPink] (0.7, 0) -- (0.7, 0.75);

  \draw[dashed, cbBlue] (0, 0.75) -- (1.71, 0.75);
  \draw[dashed, cbPink] (0, 0.75) -- (0.7, 0.75);
  
  \draw (0, 0.75) node[left, inner sep=2pt] {$1-\alpha$};
  \draw (0, 1) node[left, inner sep=2pt] {$1$};
  \draw (0, 0.1) node[left, inner sep=2pt] {$0$};
  \draw (2, -0.15) node[left, inner sep=2pt] {$q_{1-\alpha}^{\mathrm{rd}}$};
  \draw (1, -0.15) node[left, inner sep=2pt] {$q_{1-\alpha}$};
\end{tikzpicture}
\caption{Cumulative distribution functions (CDFs) of nonconformity scores using true labels ($F$, purple) and random labels ($F_{\mathrm{rd}}$, blue), along with the corresponding $1-\alpha$-quantiles ($q_{1-\alpha}$ and $q_{1-\alpha}^{\mathrm{rd}}$).}
\label{fig:sketch}
\end{figure}
Under assumption $F_{\mathrm{rd}}(\cdot)\leq F(\cdot)$, our setting reduces to \Cref{fig:sketch}. The $1-\alpha$-quantile of the nonconformity scores evaluated at the noisy labels ($\hat{q}_{1-\alpha}$) can fall into three distinct regions of the CDF domain:
\begin{itemize}
    \item[]$I_{1}=$($q_{1-\alpha}^{\mathrm{rd}}; +\infty)$: The noisy label algorithm produces unreliable labels (higher nonconformity scores) compared to random assignment at level $1-\alpha$ (violating $F_{\mathrm{rd}}(\cdot)\leq \hat{F}(\cdot)$). In this region $\bar{r}=0$. 
    \item[]$I_{2}=[q_{1-\alpha};q_{1-\alpha}^{\mathrm{rd}}]$: Noisy labels are more reliable than random but less reliable than true labels. In this region $\bar{r}=0$, randomness injection is therefore unnecessary to guarantee marginal coverage (\Cref{eq:marginal:cov:A}).
    \item[]$I_{3}=(-\infty;q_{1-\alpha})$: Noisy labels are more reliable than both random and true labels at level $1-\alpha$. In this region $\bar{r}\in (0,1]$.  
\end{itemize}
In practice, the main challenge is that true optimal treatments, and consequently $q_{1-\alpha}$, are unknown, making it difficult to distinguish between regions $I_{2}$ and $I_{3}$. However, we can identify situations falling within $I_{1}$, that do not comply with $F_{\mathrm{rd}}(\cdot)\leq \hat{F}(\cdot)$. 

\paragraph{Randomness injection mechanism illustration.}
Let $\alpha^{*} \in [0,1]$ be a target level such that the $(1-\alpha^{*})$-quantiles of the true and noisy score distributions coincide ($q_{1-\alpha^*} = \hat{q}_{1-\alpha^*}$) and $F(t) \ge \hat{F}(t)$ for all $t \geq q_{1-\alpha^{*}}$, defining the validity region (see \Cref{app:noisy_ICP}). A set-valued policy $\hat{C}_{\alpha}^{r}$ constructed via \Cref{algo:CSP} with $r=0$ preserves coverage at level $1-\alpha$ provided that $q_{1-\alpha} \leq \hat{q}_{1-\alpha}$ (i.e. $q_{1-\alpha} \in [q_{1-\alpha^*}, +\infty)$) \citep[Theorem 1]{alaa2023confmetalrpredinf}.
\begin{figure}[H]
\centering
\begin{tikzpicture}[scale=2]
  \draw[thick, ->] (-1, 0) -- (3, 0) node[right] {$x$};
  \draw[thick, ->] (0, 0) -- (0, 1.2) node[above]{}; 
  
 \draw[cbPink, thick, smooth] plot[domain=0:2.5, samples=100] 
    (\x, {1/(1 + exp(-7*(\x - 1.15)))});
\draw[cbGreen, dashed, smooth] plot[domain=0:2.5, samples=100] 
    (\x, {1/(1 + exp(-3.5*(\x - 1.2)))});
 \draw[cbGreen, thick, smooth] plot[domain=0:2.5, samples=100] 
    (\x, {1/(1 + exp(-3.5*(\x - 0.9)))});
    
  \draw[dashed, gray] (0, 0.65) -- (1.4, 0.65);
  \draw[dashed, gray] (0, 0.42) -- (1, 0.42);
  \draw[dashed, gray] (0, 0.85) -- (1.45, 0.85);
  
    \draw[dashed, gray] (1.1, 0) -- (1.1, 0.4);
     \draw[dashed, gray] (1.4, 0) -- (1.4, 0.85);

    \fill (1.1,0.42) circle (1pt); 
    \fill (1.4,0.85) circle (1pt);
  
  \draw (0, 0.65) node[left, inner sep=2pt] {$1-\alpha$};
  \draw (0, 0.42) node[left, inner sep=2pt] {$1-\alpha^{*}_{r}$};
  \draw (0, 0.85) node[left, inner sep=2pt] {$1- \alpha^{*}$};
  \draw (0, 0.1) node[left, inner sep=2pt] {$0$};
  \draw (1.2, -0.15) node[left, inner sep=2pt] {$q_{1-\alpha^{*}_{r}}$};
  \draw (1.8, -0.15) node[left, inner sep=2pt] {$q_{1-\alpha^{*}}$};
\end{tikzpicture}
\caption{Cumulative distribution functions (CDFs) of nonconformity scores using true labels ($F$, purple), noisy labels ($\hat{F}$, green) and perturbed noisy labels with $r \in (\bar{r},1] \cup \{1\}$ ($\hat{F}_{r}$, dashed green).}
\label{fig:sketch:noise}
\end{figure}

\Cref{fig:sketch:noise} illustrates a setting where noisy labels (solid green line) violate coverage guarantees (i.e.\ $q_{1-\alpha}<q_{1-\alpha^{*}}$). Intuitively, injecting randomness expands the validity region. Assuming $F_{\mathrm{rd}}(\cdot)\leq \hat{F}(\cdot)$ holds, we obtain $\hat{F}_{r}(\cdot) \leq \hat{F}(\cdot)$, and any randomness level $r\in [0,1]$. Injecting randomness at level $r$, yields $1-\alpha^{*}_{r}\in [0, 1-\alpha^{*}]$, thereby widening the validity region to $[q_{1-\alpha^{*}_{r}}, +\infty) \supseteq [q_{1-\alpha^{*}}, +\infty)$. Consequently, when $r$ satisfies \Cref{thm:coverage:r}, $q_{1-\alpha}$ falls within the expanded validity region $[q_{1-\alpha^{*}_{r}}, +\infty)$, restoring coverage as illustrated in \Cref{fig:sketch:noise} (dashed green line).

\newpage
\section{Details on numerical experiments}
\label{app:results}

\subsection{Additional metrics for set-valued policy evaluation (\Cref{subsec:synthetic:data})}
\label{def:coverage_notions}
Let $\{(X_{i}, \Pi^{\star}(X_{i}))\}_{i \in [n]}$ be a sequence of $n$ i.i.d.\ observations, where $\Pi^{\star}(X_{i}) \subseteq \mathcal{A}$ denotes the ground-truth set of optimal treatments for covariate $X_{i}$. For a prediction set mapping $C: \mathcal{X} \to \mathcal{P}(\mathcal{A})$, we define:
\begin{itemize}
    \item \textbf{Relaxed coverage:} 
    $\frac{1}{n}\sum_{i=1}^{n} \bone\{\Pi^{\star}(X_{i}) \cap C(X_{i}) \neq \emptyset\}$,
    \item \textbf{Strict coverage:} $\frac{1}{n}\sum_{i=1}^{n} \bone\{\Pi^{\star}(X_{i}) \subseteq C(X_{i})\}$,
    \item \textbf{Exact match:} $\frac{1}{n}\sum_{i=1}^{n} \bone\{\Pi^{\star}(X_{i})= C(X_{i})\}$,
    \item \textbf{Mean cardinality:} $\frac{1}{n}\sum_{i=1}^{n} |C(X_{i})|$.
\end{itemize}

\subsection{Training details}
\label{app:estimation:details}

\begin{enumerate}
    \item  \textbf{Baseline policies:} We trained eleven baseline policies: linear models with and without interaction terms \citep{nelder1972generalized}, regression forests \citep{athey2019generalized}, multi armed causal forest (\texttt{MACF}), and policy trees \citep{sverdrup2020policytree}. We also implemented doubly robust Q-learning via the \texttt{polle} package \citep{nordland2022policy} using three different Q-models: a linear model, XGBoost \citep{chen2016xgboost}, support vector machines \citep{karatzoglou2004kernlab} and random forests \citep{wright2017ranger}, paired with a generalized random forest \citep{athey2019generalized} for the propensity model. We also generated a policy via the Super Learner ensemble framework \citep{van2007super, SuperLearner}, whose base library included mean imputation, generalized linear models, random forests \citep{wright2017ranger}, gradient-boosted trees (XGBoost) \citep{chen2016xgboost}, and support vector machines (kSVM) \citep{karatzoglou2004kernlab}. Finally, we created an ensemble baseline (Aggregation) by sampling treatments from an unweighted mixture distribution across the eleven individual baselines.

    
    \item  \textbf{Greatest lower bound:} Across all numerical experiments, we estimated $a\mapsto \mu(\cdot, a)$ using \textit{(i)} a linear model with interaction terms \citep{nelder1972generalized} and \textit{(ii)} a regression forest \citep{athey2019generalized, tibshirani2018package}. The corresponding upper- and lower-bounds were then derived using the delta method for the linear model and the infinitesimal jackknife for the regression forest.
    \item \textbf{Conformal policy learning:} The procedure involves estimation at steps 1 and 2 in \Cref{subsec:Conformal policy learning} 
    \begin{enumerate}
        \item[(1)] \textbf{Black-box label generation:} we generated the noisy labels $\hat{A}_{i}^{*}=\mathcal{B}(\mathcal{D}_{b})(X_{i})$ using optimal treatment regime (OTR) estimators trained via five distinct methods for the synthetic settings: 
        \begin{itemize}
        \item \textbf{Two high-performing models:} \texttt{MACF} alongside doubly robust Q-learning with a linear Q-function (for the first DGP), and \texttt{MACF} alongside Q-learning with a regression forest (for the second DGP). For these methods, the randomness level was set to $r\in \{0,0.1\}$ for the first DGP and $r\in \{0,0.15\}$ for the second. 
        \item \textbf{Two lower-performing baselines:} Doubly robust Q-learning with a support vector machine (kSVM) Q-function, and a standard linear model. For these methods the randomness level was fixed to $r=0$.
        \item \textbf{Aggregation}: An unweighted mixture of the eleven baseline policies mentioned previously. The randomness level was set to $r=0$.
        \end{itemize}
       For the IVF application, noisy labels were generated using \texttt{MACF} whereas for the trauma care application, they were generated using a Super Learner ensemble comprising mean imputation, GLMs, random forests, XGBoost, and kSVM.
        \item[(2)] \textbf{Nonconformity score:} The conditional mean $\hat{\mu}$ used to define the empirical nonconformity score $\hat{s}$ in \Cref{eq:non:conformity:score:reg} was estimated via the super learner ensemble method \citep{van2007super, SuperLearner} for the synthetic simulations and application to IVF data.  The \texttt{SuperLearner}'s library included mean imputation, generalized linear models,  random forests, XGBoost, and kSVM. For the trauma care application we used a probability forest \cite{athey2019generalized}. 
    \end{enumerate}
    \item \textbf{SPV estimation on trauma care application:} We estimated both the conditional mean and propensity score using probability forests.
\end{enumerate}

\paragraph{Split proportions.} The conformal procedure requires partitioning data into subsets indexed by $\mathcal{I}_{b}, \mathcal{I}_{\mathrm{train}}$ and $\mathcal{I}_{\mathrm{cal}}$. Throughout the simulation study, these subsets are allocated in proportions of 0.4, 0.4, and 0.2, respectively. For the trauma care application, the training data were split into three disjoint subsets dedicated to policy learning, nuisance estimation (for SPV), and evaluation, in proportions of 0.5, 0.3, and 0.2, respectively.

\paragraph{Seed.} To ensure reproducibility, the random seed was set to 2026.

\subsection{Additional details on synthetic setting \Cref{subsec:synthetic:data}}
\label{app:synthetic:data}
\textbf{Synthetic data generation:} We generate synthetic observational data with linear decision boundary characterizing two optimal treatment regions. We sample features $X\in \mathbb{R}^{5}$ from a multivariate Gaussian distribution $X \sim \mathcal{N}(0,\mathcal{I}_{5})$. We consider four available treatments ($A\in \{1,\ldots,4\}$) and generate potential outcomes $Y(a)|X=x$,  according to $\mu(x,a) + 0.1\epsilon$, with i.i.d.\ noise $\epsilon \sim \mathcal{N}(0,1)$. The conditional mean, 
\begin{equation}
\label{eq:conditional:mean}
\mu(x,a) = b(x) + 
    \begin{cases}
        2\cdot \bone\{(x_1 + x_2) < 0.5\} &\text{ if } a\in \{1,2\}\\
        3\cdot\bone\{(x_{1}+x_{2} \geq 0.5)\} &\text{ if } a = 3. 
    \end{cases}
\end{equation} 
decomposes into a baseline effect  $b(x) = X\beta$ with $\beta = [1, -0.5, 1, -0.5, 1]^{\top}$ and a treatment interaction term. A linear hyperplane, $(x_1 + x_2)= 0.5$, partitions the potential outcomes (see \Cref{subfig:synthetic:linear}). Observations below the plane yield higher values for $a\in \{1,2\}$, whereas those above result in higher values for $a = 3$. Treatment $4$ is consistently suboptimal, characterized by a null interaction term. 
 
Treatment assignments ($A$) are sampled from the distribution $\pi_{b}(X)\in [0,1]^{4}$, designed to reflect sub-optimal clinical decision-making. We define two preference states, $\beta_{\mathrm{low}}=[8,8,5,7]^\top$ and $\beta_{\mathrm{high}}=[4,4,8,7]^{\top}$, which prioritize treatments $\{1,2\}$ and $\{3\}$, respectively. The transition between these vectors is modulated by a logistic weighting function $w(x)\mapsto \sigma(x_{1}+x_{2}-0.5)$, where $z\mapsto  \sigma(z)=(1+e^{-z})^{-1}$. This models a smooth linear gradient across the decision boundary $x_{1}+x_{2}=0.5$, defining:
\begin{equation*}
\label{eq:behavioral:policy}
    \pi_{b}(x)=\frac{w(x)\beta_{\mathrm{low}}+ (1-w(x))\beta_{\mathrm{high}}}{\sum_{k=1}^{4}w(x)\beta_{\mathrm{low},k}+ (1-w(x))\beta_{\mathrm{high},k}}.
\end{equation*}

\subsection{Additional figures from \Cref{subsec:synthetic:data}}
\label{app:additional:figures}

\begin{figure}[H]
    \begin{minipage}{0.48\linewidth}
     \centering
        \includegraphics[width=\linewidth]{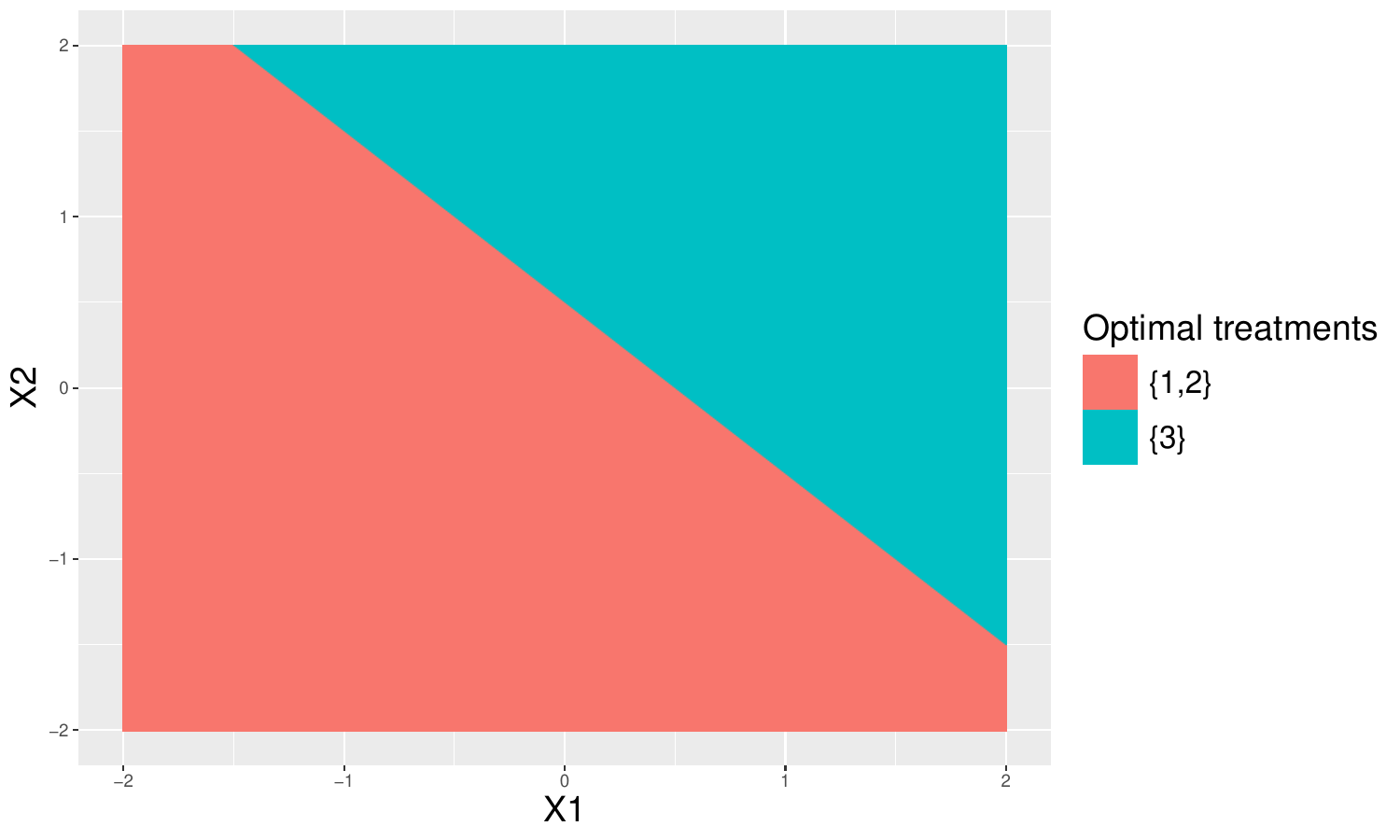}
    \subcaption{Synthetic data}
    \label{subfig:synthetic:linear}
    \end{minipage}
    \begin{minipage}{0.48\linewidth}
     \centering
        \includegraphics[width=\linewidth]{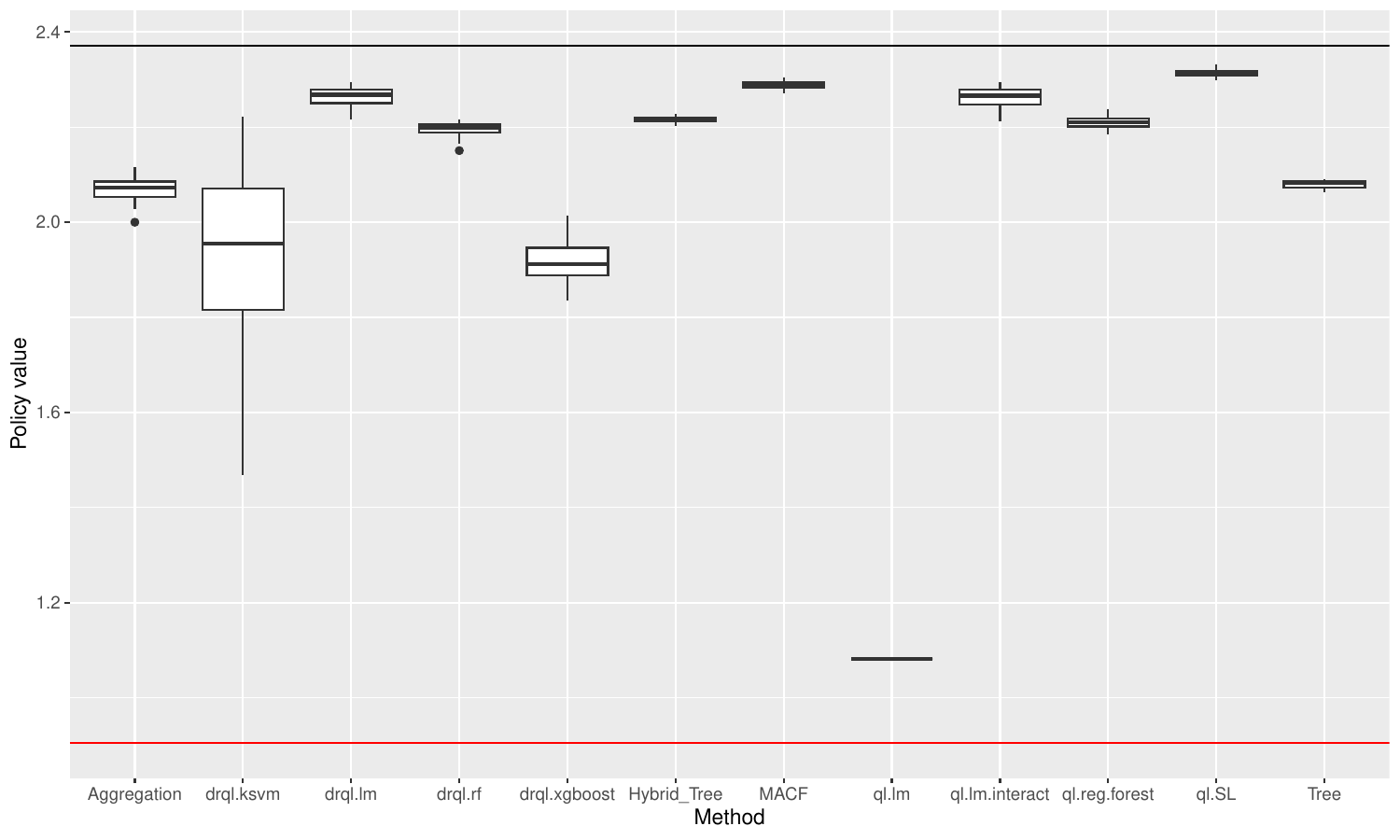}
        \subcaption{Policy values for baseline policy learning algorithms}
        \label{subfig:naive:policies}
    \end{minipage}
    \caption{(a) Distribution of optimal treatment assignments by feature values. (b) Oracle policy values for eleven baseline policy learning methods and their aggregated policy (Aggregation). Horizontal lines indicate optimal policy (black) and random treatment allocation (red) values; the latter serves as a minimal benchmark.}
\end{figure}

\begin{figure}[H]
    \centering
        \includegraphics[width=\linewidth]{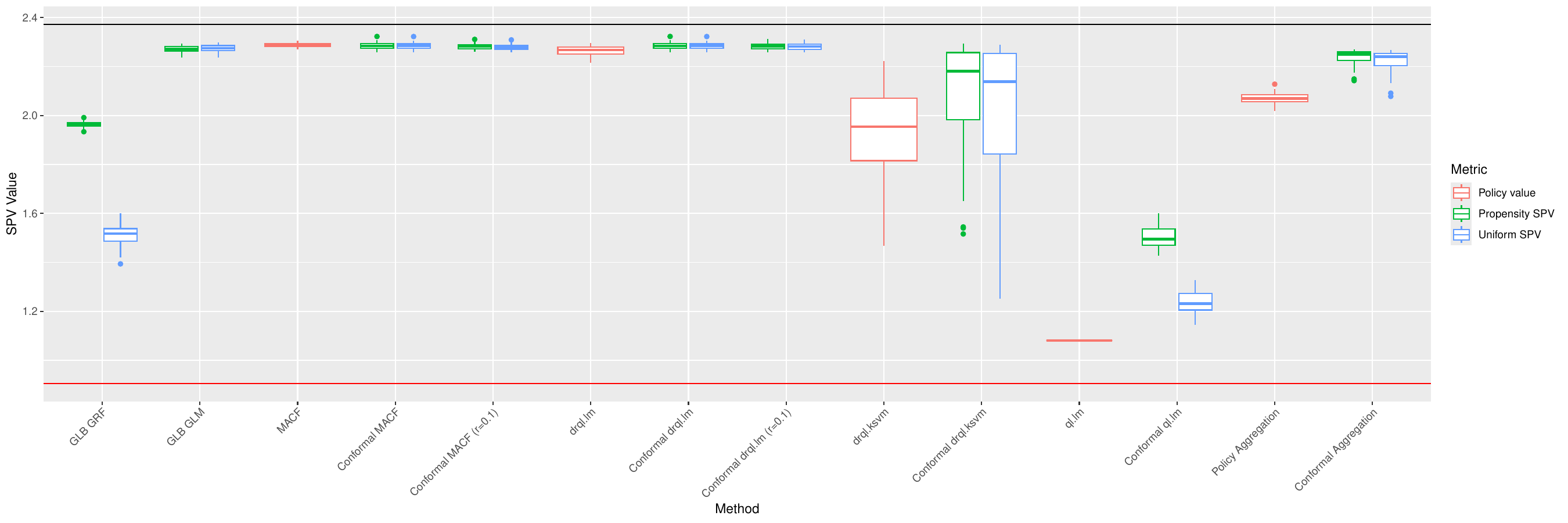}
    \caption{Boxplots comparing uniform (blue) and propensity-based (green) SPVs for different set-valued policies against the values of different baseline policies (red). Horizontal lines indicate optimal policy (black) and random treatment allocation (red) values.}
     \label{fig:spv:boxplot}
\end{figure}

\begin{table}[h]
\centering
\small
\setlength{\tabcolsep}{1pt}
\begin{table}[H]
\centering\begingroup\fontsize{9}{11}\selectfont

\begin{tabular}{lcccccccccccccc}
\toprule
\multicolumn{1}{c}{ } & \multicolumn{2}{c}{\textbf{GLB}} & \multicolumn{3}{c}{\textbf{MACF}} & \multicolumn{3}{c}{\textbf{drql.lm}} & \multicolumn{2}{c}{\textbf{drql.ksvm}} & \multicolumn{2}{c}{\textbf{ql.lm}} & \multicolumn{2}{c}{\textbf{Aggregation}} \\
\cmidrule(l{3pt}r{3pt}){2-3} \cmidrule(l{3pt}r{3pt}){4-6} \cmidrule(l{3pt}r{3pt}){7-9} \cmidrule(l{3pt}r{3pt}){10-11} \cmidrule(l{3pt}r{3pt}){12-13} \cmidrule(l{3pt}r{3pt}){14-15}
\textbf{Set Type} & GRF & GLM & Policy & Conf. & Conf.$_{r=0.1}$ & Policy & Conf. & Conf.$_{r=0.1}$ & Policy & Conf. & Policy & Conf. & Policy & Conf.\\
\midrule
$\Pi^{\star}(X_{i})$ = \{1, 2\} & 0.31 & 0.81 & \textbf{0.00} & 0.70 & 0.84 & \textbf{0.00} & 0.64 & 0.81 & \textbf{0.00} & 0.60 & \textbf{0.00} & 0.00 & \textbf{0.00} & 0.82\\
$\Pi^{\star}(X_{i})$ = \{3\} & 0.18 & 1.00 & 0.99 & 0.99 & 0.97 & 1.00 & 0.99 & 0.98 & 1.00 & 0.75 & 1.00 & 0.29 & 0.97 & 0.90\\
Mean cardinality & \textbf{2.64} & 1.62 & 1.00 & 1.49 & 1.63 & 1.00 & 1.45 & 1.59 & 1.00 & 2.15 & 1.00 & \textbf{3.45} & 1.00 & 1.80\\
\addlinespace
Strict coverage \{1, 2\} & \textbf{1.00} & \textbf{0.87} & 0.00 & \textbf{0.72} & \textbf{0.89} & 0.00 & \textbf{0.65} & \textbf{0.84} & 0.00 & \textbf{0.97} & 0.00 & \textbf{1.00} & 0.00 & \textbf{0.98}\\
Strict coverage \{3\} & 1.00 & 1.00 & 0.99 & 1.00 & 1.00 & 1.00 & 1.00 & 1.00 & 1.00 & 1.00 & 1.00 & 1.00 & 0.97 & 1.00\\
Total strict coverage & \textbf{1.00} & \textbf{0.92} & 0.35 & 0.82 & \textbf{0.93} & 0.35 & 0.78 & \textbf{0.90} & 0.35 & \textbf{0.98} & 0.35 & \textbf{1.00} & 0.34 & \textbf{0.99}\\
Marginal coverage & 1.00 & 0.97 & 0.96 & 0.97 & 0.98 & 0.95 & 0.97 & 0.98 & 0.78 & 0.99 & 0.35 & 1.00 & 0.86 & 0.99\\
\addlinespace
Uniform SPV & \textbf{1.51} & 2.27 & 2.29 & 2.28 & 2.28 & 2.26 & 2.28 & 2.28 & 1.94 & 1.99 & 1.08 & \textbf{1.24} & 2.07 & 2.22\\
Propensity SPV & \textbf{1.96} & 2.27 & 2.29 & 2.28 & 2.28 & 2.26 & 2.28 & 2.28 & 1.94 & 2.07 & 1.08 & \textbf{1.50} & 2.07 & 2.24\\
\bottomrule
\end{tabular}
\endgroup{}
\end{table}
\caption{Comparison of GLB, baseline (Pol.) and set-
valued (Conf.) policies ($r\in \{0,0.1\}$). Metrics include conditional exact match, mean cardinality, conditional strict coverage, total strict/relaxed coverage, and  uniform/propensity set policy values (SPV).}
\label{tab:exact_match}
\end{table}

\paragraph{Remarks: }
For high-performing baselines, derived conformal set-valued policies maintain comparable overall value (\Cref{fig:spv:boxplot,tab:exact_match}, \texttt{MACF} and drql.lm columns). Conversely, for weaker baselines, conformalization actively improves performance (\Cref{fig:spv:boxplot,tab:exact_match}, drql.ksvm and ql.lm columns). Additionally, propensity SPV consistently outperforms uniform SPV, which penalizes suboptimal treatments and consequently larger prediction sets (\Cref{fig:spv:boxplot,tab:exact_match}, rows 3, 8, and 9).

As in \Cref{tab:reduced}, baseline policy learning methods fail to capture the indistinguishability of treatments $\{1,2\}$ (\Cref{tab:exact_match}, row 1, Policy columns). In contrast, both GLB and conformal set-valued policies substantially improve exact match rates and conditional coverage for these indistinguishable sets (\Cref{tab:exact_match}, row 1 and 4). While all set-valued methods achieve marginal coverage (\Cref{tab:exact_match}, row 6) regardless of $r$, strict coverage for conformalized top-performing baselines is preserved only when incorporating randomness (\Cref{tab:exact_match}, row 6, Conf.$_{r=0.1}$ columns). 
\vspace{2.5cm}

\paragraph{Double robustness of AIPW and TMLE estimators. }
We generated $n=10{,}000$ training samples and $n_{\mathrm{test}}=5{,}000$ test samples, and split training data into two folds. The first fold was used to train a set-valued policy $\hat{C}_{\alpha}$ via GLB alongside with a linear model with interaction terms at a fixed significance level $\alpha=0.1$. The second fold was used to fit the mis-specified nuisance models for evaluation. We evaluated the AIPW, and TMLE estimators of the uniform and propensity SPV, under four nuisance specification regimes: \textit{(i)} both nuisances correctly specified using oracles ($\mu$ and $e$), \textit{(ii)} mis-specified outcome model and oracular propensity score, \textit{(iii)} oracular outcome model and mis-specified propensity score, and \textit{(iv)} both nuisances mis-specified. 

The conditional mean $(x,a)\mapsto \mu(x,a)$ (and therefore $m$), was estimated using a Super Learner ensemble comprising mean imputation, GLMs, random forests, XGBoost, and kSVM. Propensity scores $(x,a) \mapsto e(a|x)$ (thus $g_{C}$), were estimated via probability forests. Mis-specification was induced by withholding the first covariate during training. Asymptotic confidence intervals were computed as $x \mapsto [x\pm q^{\mathrm{norm}}_{1-\alpha/2} \sqrt{\hat{\sigma}^{2}(C_{\alpha})}]$, where $q^{\mathrm{norm}}_{1-\alpha/2}$ is the standard normal $1-\alpha/2$-quantile and $\hat{\sigma}^{2}(C_{\alpha})$ is the estimated variance of the corresponding efficient influence function (see \Cref{app:evaluation_fixed_choice_function,thm:double_robustness,thm:double_robustness-tmle}).

\Cref{fig:dr} highlights the double robustness of our AIPW and TMLE estimators, whose performance closely matches the oracle value when both nuisance models are correctly specified \textit{(i)} or when only one nuisance model is mis-specified \textit{(ii)} and \textit{(iii)}.

\begin{figure}[H]
    \centering
    \includegraphics[width=\linewidth]{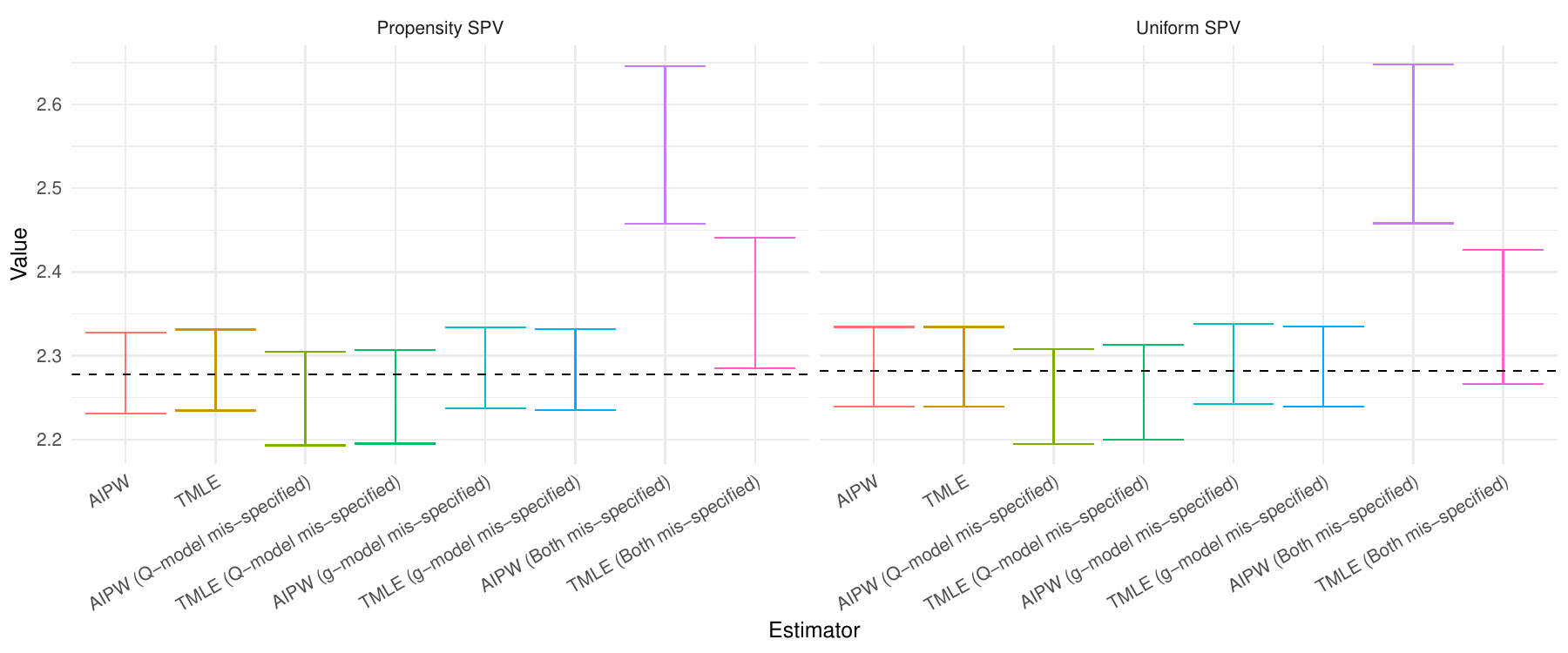}
    \caption{Asymptotic confidence intervals for the AIPW, and TMLE SPV estimators across four nuisance specification regimes: \textit{(i)} both nuisances correctly specified, \textit{(ii)} outcome model mis-specified, \textit{(iii)} propensity score mis-specified and \textit{(iv)} both nuisances mis-specified. The dashed black line denotes the true oracular policy value.}
    \label{fig:dr}
\end{figure}

\subsection{Additional DGP}
\label{app:additional:setting}
\paragraph{Synthetic data generation.} We generate synthetic observational data with tree-shaped decision boundary characterizing four optimal treatment regions. We sample features $X\in \mathbb{R}^{5}$ from a multivariate Gaussian distribution $X \sim \mathcal{N}(0,\mathcal{I}_{5})$. We consider four available treatments ($A\in \{1,\ldots,4\}$) and generate potential outcomes $Y(a)|X=x$, according to $\mu(x,a) + 0.1\epsilon$, with i.i.d.\ noise $\epsilon \sim \mathcal{N}(0,1)$. The conditional mean, 
\begin{equation}
\label{eq:conditional:mean}
\mu(x,a) = b(x) + 
    \begin{cases}
        3\cdot \bone\{x_1, x_{2} \leq 0\} &\text{ if } a\in \{1,2\}\\
        3\cdot \bone\{x_1, x_{2} > 0 \} &\text{ if } a\in \{2,3\}\\
        2\cdot \bone\{x_1 \leq 0,x_2 > 0\} &\text{ if } a = 1 \\
        2\cdot \bone\{x_1 > 0, x_2 \leq 0\} &\text{ if } a = 3 \\
    \end{cases}
\end{equation} 
decomposes into a baseline effect (same used in \Cref{app:synthetic:data}) and a treatment interaction term. The hyperplane partitions the potential outcomes into four regions (\Cref{subfig:synthetic:tree}) where either \{1,2\}, $\{2,3\}$, $\{1\}$ or $\{3\}$ yield higher values. Treatment $4$ is once again consistently suboptimal, characterized by a null interaction term. 
 
Treatment assignments ($A$) are sampled from the distribution $\pi_{b}(X)\in [0,1]^{4}$, designed to reflect sub-optimal clinical decision-making. We define four preference states $\beta_{1}=[10, 10,  5,  8]^{\top}$, $\beta_{2}= [5, 10, 10,  8]^{\top}$, $\beta_{3}= [10,  5,  5,  8]^{\top}$ and $\beta_{4}=[5,  5, 10,  8]^{\top}$  prioritizing treatment sets $\{1,2\}$, $\{2,3\}$, $\{1\}$ and $\{3\}$, respectively. To smoothly weight these preference states across the four regions, we define mixture weights
$w_{k}(x)$ using the sigmoid $\sigma(z)=(1+e^{-z})^{-1}$: 
\begin{align*}
w_{1}(x) &= (1 - \sigma(x_{1}))(1 - \sigma(x_{2})), & w_{2}(x)= \sigma(x_{1})\sigma(x_{2}), \\
w_{3}(x) &= (1 - \sigma(x_{1}))\sigma(x_{2}), & w_{4}(x)= \sigma(x_{1})(1 - \sigma(x_{2})).
\end{align*}
The propensity distribution is then given by the normalized mixture:
\begin{equation*}
\label{eq:behavioral:policy}
\pi_{b}(x) = \sum_{k=1}^{4}  \frac{w_k(x)}{\sum_j w_j(x)}  \beta_k.
\end{equation*}

\begin{figure}[H]
    \centering
    \begin{minipage}{0.48\linewidth}
     \includegraphics[width=\linewidth]{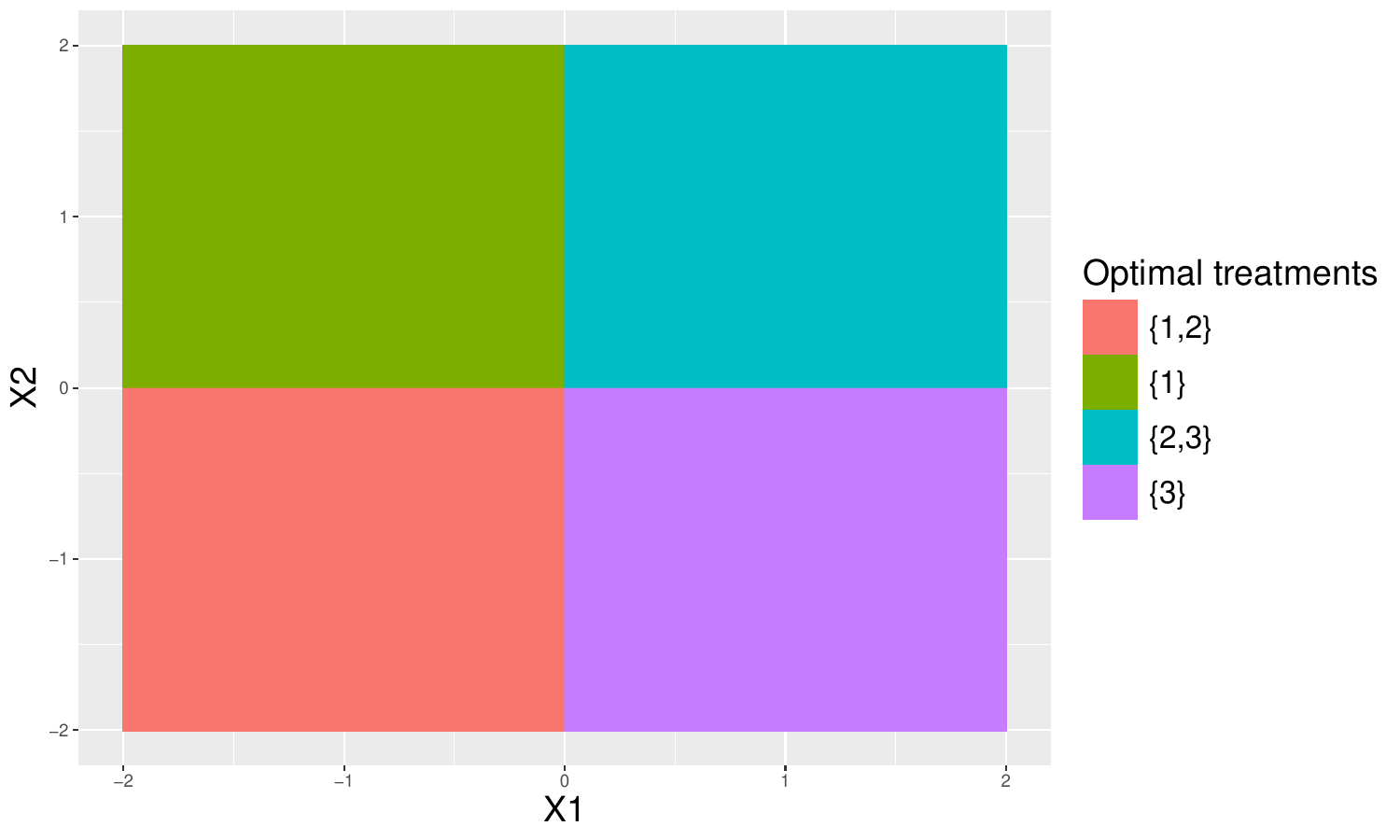}
        \subcaption{Synthetic data}
    \label{subfig:synthetic:tree}
    \end{minipage}
        \begin{minipage}{0.48\linewidth}
        \includegraphics[width=\linewidth]{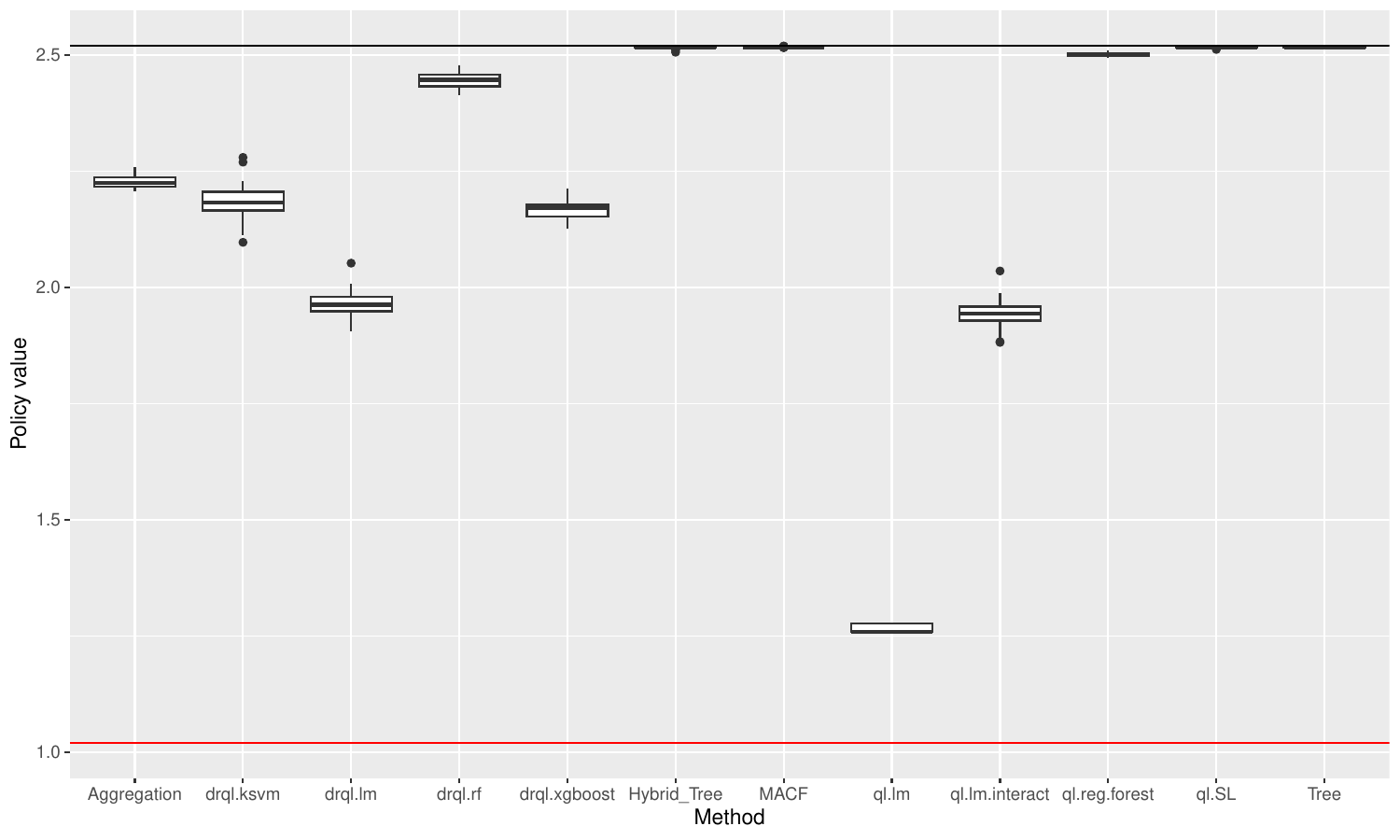}
        \subcaption{Policy values for baseline policy learning algorithms}
        \label{subfig:naive:policies:tree}
    \end{minipage}
    \caption{(a) Distribution of optimal treatment assignments by feature values. (b)  Oracle policy values for eleven baseline policy learning methods and their aggregated policy (Aggregation). Horizontal lines indicate optimal policy (black) and random treatment allocation (red) values; the latter serves as a minimal benchmark}
\end{figure}

\textbf{Estimation details. } Following the procedure from the previous synthetic experiment (\Cref{subsec:synthetic:data,app:estimation:details}), we trained GLB, the conformal procedure, and the eleven baselines (see \Cref{subfig:naive:policies:tree}) with one exception: we set $r\in \{0,0.15\}$ in the conformal framework for the top-performing noisy label generation baselines.

\begin{figure}[H]
    \centering
        \includegraphics[width=\linewidth]{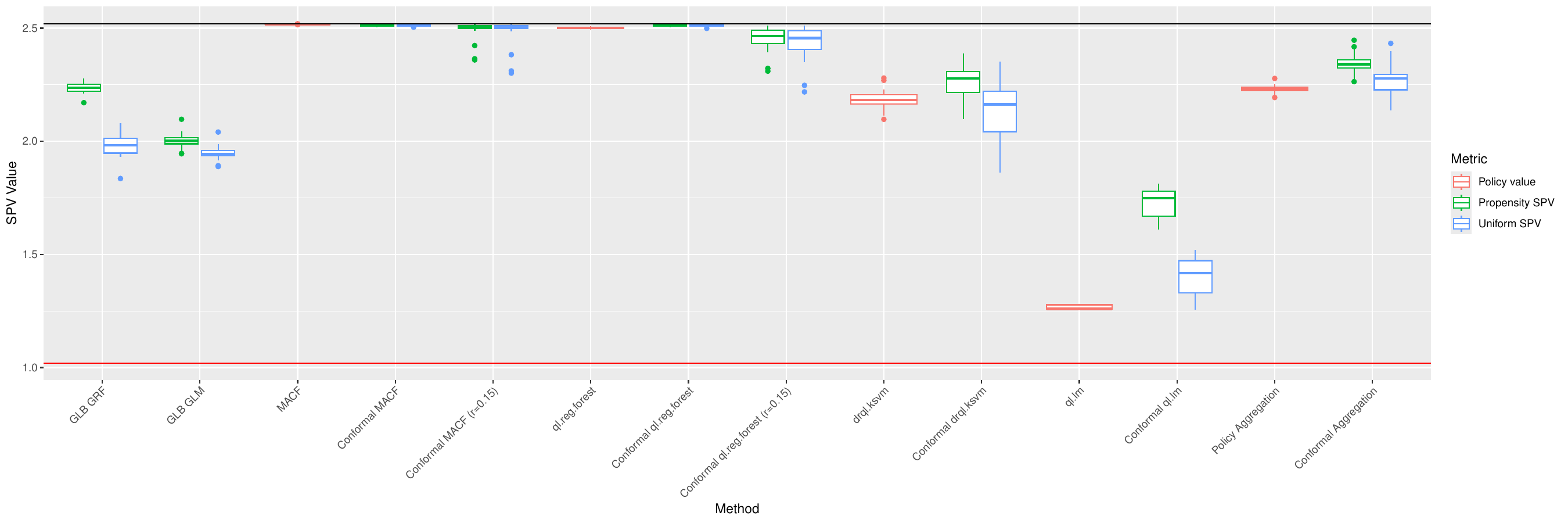}
    \caption{Boxplots comparing uniform (blue) and propensity-based (green) SPVs for different set-valued policies against the values of different baseline policies (red). Horizontal lines indicate optimal policy (black) and random treatment allocation (red) values.}
    \label{fig:spv:boxplot:tree}
\end{figure}

\paragraph{Remarks: }
Conformalization retains baseline-level value when backed by competitive underlying models (\Cref{fig:spv:boxplot:tree,tab:exact_match:tree}, ql.reg.forest and \texttt{MACF} columns), while providing a marked performance uplift for lower-performing base policies (\Cref{fig:spv:boxplot:tree,tab:exact_match:tree}, drql.ksvm and ql.lm columns). 

As in \Cref{tab:exact_match}, baseline policy learning methods fail to capture the indistinguishability of multiple treatments ($\{1,2\}$ and $\{2,3\}$) (\Cref{tab:exact_match:tree}, rows 1 and 3, Policy columns). In contrast, both GLB and conformal set-valued policies substantially improve exact match rates and strict conditional coverage for these indistinguishable sets (\Cref{tab:exact_match}, rows 1 and 6). While all set-valued methods achieve marginal coverage regardless of $r$ (\Cref{tab:exact_match:tree}, row 11), strict coverage for conformalized ql.reg.forest is preserved only when incorporating randomness (\Cref{tab:exact_match:tree}, row 10, Conf.$_{r=0.15}$ column). In this setting, GLB performs worse than both baseline and conformal set-valued policies. In particular, GLB GLM is limited, likely because the estimator fails to capture the underlying nuisance pattern.

\begin{table}[h]
\centering
\small
\setlength{\tabcolsep}{1pt}
\begin{table}[H]
\centering\begingroup\fontsize{9}{11}\selectfont
\begin{tabular}{lcccccccccccccc}
\toprule
\multicolumn{1}{c}{ } & \multicolumn{2}{c}{\textbf{GLB}} & \multicolumn{3}{c}{\textbf{MACF}} & \multicolumn{3}{c}{\textbf{ql.reg.forest}} & \multicolumn{2}{c}{\textbf{drql.ksvm}} & \multicolumn{2}{c}{\textbf{ql.lm}} & \multicolumn{2}{c}{\textbf{Aggregation}} \\
\cmidrule(l{3pt}r{3pt}){2-3} \cmidrule(l{3pt}r{3pt}){4-6} \cmidrule(l{3pt}r{3pt}){7-9} \cmidrule(l{3pt}r{3pt}){10-11} \cmidrule(l{3pt}r{3pt}){12-13} \cmidrule(l{3pt}r{3pt}){14-15}
\textbf{Set Type} & GRF & GLM & Policy & Conf. & Conf.$_{r=0.15}$ & Policy & Conf. & Conf.$_{r=0.15}$ & Policy & Conf. & Policy & Conf. & Policy & Conf.\\
\midrule
$\Pi^{\star}(X_{i})$ = \{1, 2\} & 0.40 & 0.61 & \textbf{0.00} & 0.42 & 0.94 & \textbf{0.00} & 0.27 & 0.98 & \textbf{0.00} & 0.98 & \textbf{0.00} & 0.32 & \textbf{0.00} & 0.99\\
$\Pi^{\star}(X_{i})$ = \{1\} & 0.63 & 0.11 & 1.00 & 0.99 & 0.95 & 0.99 & 0.99 & 0.86 & 0.70 & 0.30 & 0.43 & 0.00 & 0.78 & 0.53\\
$\Pi^{\star}(X_{i})$ = \{2, 3\} & 0.90 & 0.63 & \textbf{0.00} & 0.37 & 0.92 & \textbf{0.00} & 0.23 & 0.97 & \textbf{0.00} & 0.98 & \textbf{0.00} & 0.46 & \textbf{0.00} & 0.99\\
$\Pi^{\star}(X_{i})$ = \{3\} & 0.14 & 0.06 & 1.00 & 0.99 & 0.94 & 0.98 & 0.99 & 0.85 & 0.66 & 0.39 & 0.57 & 0.00 & 0.78 & 0.59\\
Mean cardinality & \textbf{2.15} &\textbf{2.08} & 1.00 & 1.20 & 1.51 & 1.00 & 1.13 & 1.59 & 1.00 & 2.01 & 1.00 & \textbf{3.41} & 1.00 & 1.80\\
\addlinespace
Strict coverage \{1, 2\} & \textbf{1.00} & \textbf{0.70} & 0.00 & \textbf{0.42} & \textbf{0.94} & 0.00 & \textbf{0.27} & \textbf{0.99} & 0.00 & \textbf{1.00} & 0.00 & \textbf{1.00} & 0.00 & 1.00\\
Strict coverage \{1\} & 1.00 & 0.94 & 1.00 & 1.00 & 1.00 & 0.99 & 0.99 & 1.00 & 0.70 & 1.00 & 0.43 & 1.00 & 0.78 & 1.00\\
Strict coverage \{2, 3\} & \textbf{1.00} & \textbf{0.73} & 0.00 & \textbf{0.37} & \textbf{0.92} & 0.00 & \textbf{0.23} & \textbf{0.98} & 0.00 & \textbf{1.00} & 0.00 & \textbf{1.00} & 0.00 & \textbf{1.00}\\
Strict coverage \{3\} & 1.00 & 0.92 & 1.00 & 0.99 & 1.00 & 0.98 & 0.99 & 1.00 & 0.66 & 1.00 & 0.57 & 1.00 & 0.78 & 1.00\\
Total strict coverage & \textbf{1.00} & 0.82 & 0.50 & 0.69 & \textbf{0.97} & 0.49 & 0.62 & \textbf{0.99} & 0.34 & \textbf{1.00} & 0.25 & \textbf{1.00} & 0.39 & \textbf{1.00}\\
Marginal coverage & 1.00 & 0.97 & 1.00 & 1.00 & 1.00 & 0.99 & 1.00 & 1.00 & 0.83 & 1.00 & 0.50 & 1.00 & 0.87 & 1.00\\
\addlinespace
Uniform SPV & \textbf{1.99} & \textbf{1.91} & 2.52 & 2.51 & 2.49 & 2.50 & 2.51 & 2.44 & 2.18 & 2.13 & 1.27 &  \textbf{1.40} & 2.23 & 2.27\\
Propensity SPV & \textbf{2.24} & \textbf{2.13} & 2.52 & 2.51 & 2.49 & 2.50 & 2.51 & 2.46 & 2.18 & 2.26 & 1.27 & \textbf{1.73} & 2.23 & 2.34\\
\bottomrule
\end{tabular}
\endgroup{}
\end{table}
\caption{Comparison of GLB, baseline (Pol.) and set-valued (Conf.) policies ($r\in \{0,0.15\}$). Metrics include conditional exact match, mean cardinality, conditional strict coverage, total strict/relaxed coverage, and  uniform/propensity set policy values (SPV).}
\label{tab:exact_match:tree}
\vspace{-1cm}
\end{table}

\paragraph{Double robustness of AIPW and TMLE estimators.}
\begin{figure}[H]
    \centering
    \includegraphics[width=\linewidth]{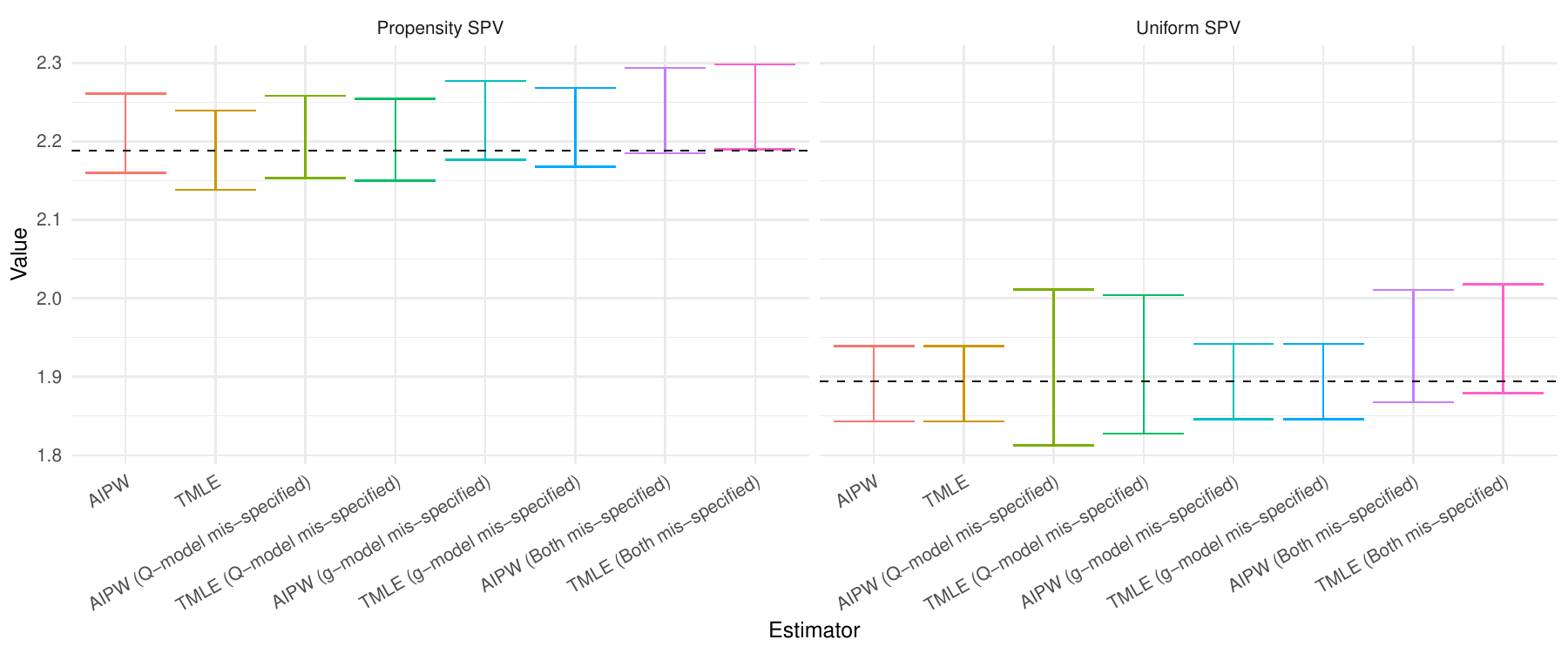}
    \caption{Asymptotic confidence intervals for the AIPW, and TMLE SPV estimators across four nuisance specification regimes: \textit{(i)} both nuisances correctly specified, \textit{(ii)} outcome model mis-specified, \textit{(iii)} propensity score mis-specified and \textit{(iv)} both nuisances mis-specified. The dashed black line denotes the true oracular policy value.}
    \label{fig:dr:tree}
\end{figure}

We followed the procedure from the previous synthetic setting with two modifications. First, we fit GLB using a regression forest to estimate the upper-and lower-bounds of the conditional mean, maintaining a fixed significance level $\alpha=0.1$. Second, model mis-specification was induced by omitting the first covariate for the outcome model, as before, and replacing the first covariate with its sign for the propensity score.

\Cref{fig:dr:tree} highlights as well the double robustness of our AIPW and TMLE estimators, whose performance closely matches the oracle value when both nuisance models are correctly specified \textit{(i)} or when only one nuisance model is
mis-specified \textit{(ii)} and \textit{(iii)}.

\subsection{Additional figures for IVF application (\Cref{subsec:ivf})}
\label{app:more:fig:elixir}

\begin{figure}[H]
\centering
\includegraphics[width=0.65\linewidth]{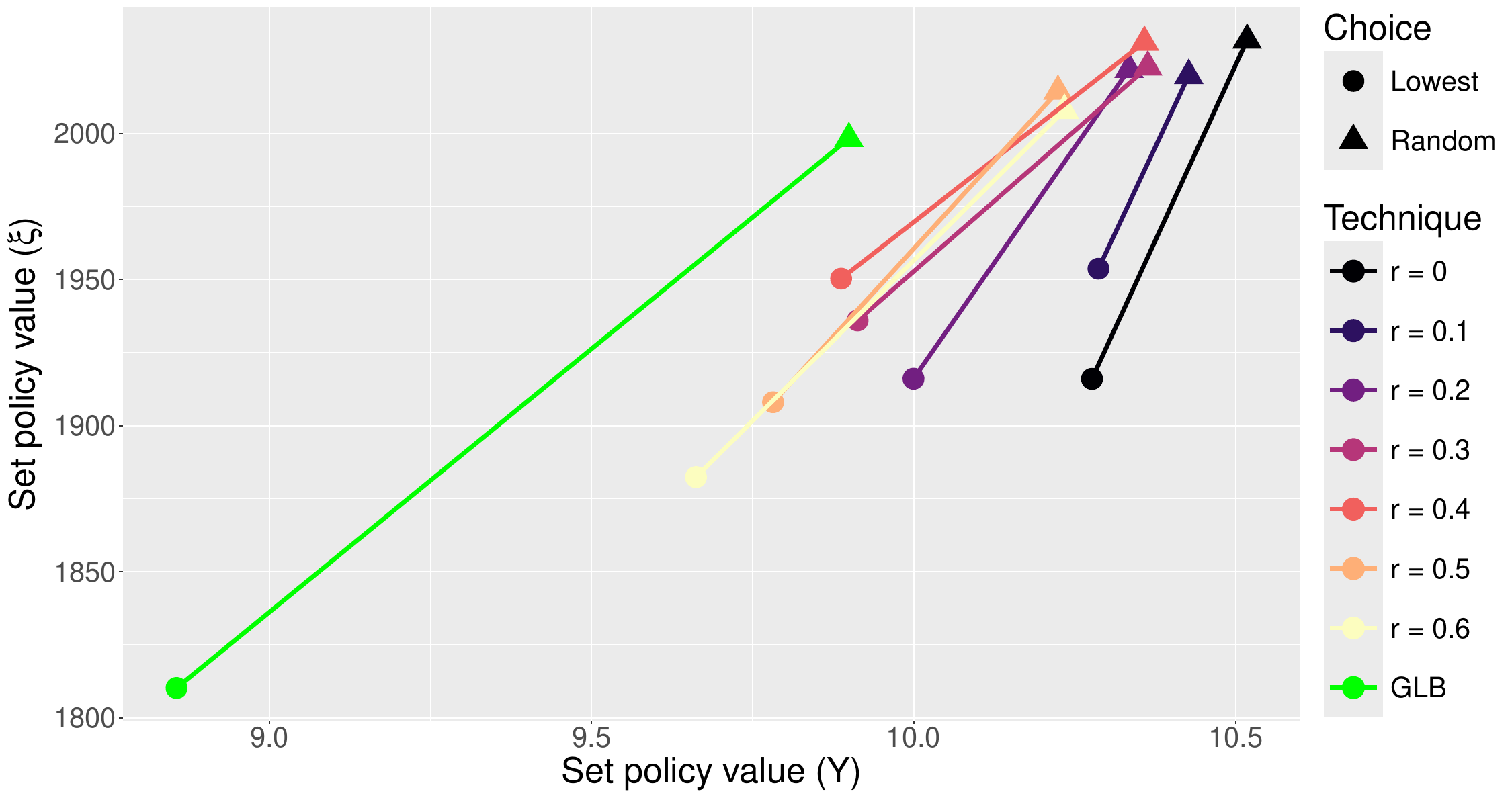}
\caption{Set-policy values comparison ($Y$ vs. $\xi$). Set-policy values for $Y$ (x-axis) and $\xi$ (y-axis) across two fixed choice functions: $\delta_{\mathrm{lower}}$ (points) and $\delta_{\mathrm{unif}}$ (triangles), $\alpha=0.1$.}
\label{fig:Y_Xi}
\vspace{-0.5cm}
\end{figure}

Delegating the final decision ($\delta$) to the decision-maker facilitates the accommodation of secondary objectives. In the IVF dataset, the secondary outcome $\xi$, represents estradiol levels, where higher values increase the risk of ovarian hyper stimulation syndrome (OHSS). While uniform sampling ($\delta_{\mathrm{unif}}$) maximizes yield on $Y$, the $\delta_{\mathrm{lower}}$ strategy facilitates control over estradiol levels $\xi$, with only a marginal reduction in $Y$ (see \Cref{fig:Y_Xi}). 

\subsection{Additional figures for trauma care application (\Cref{subsec:traumacare})}
\label{app:traumatrix}

\begin{figure}[H]
    \centering    \includegraphics[width=0.8\linewidth]{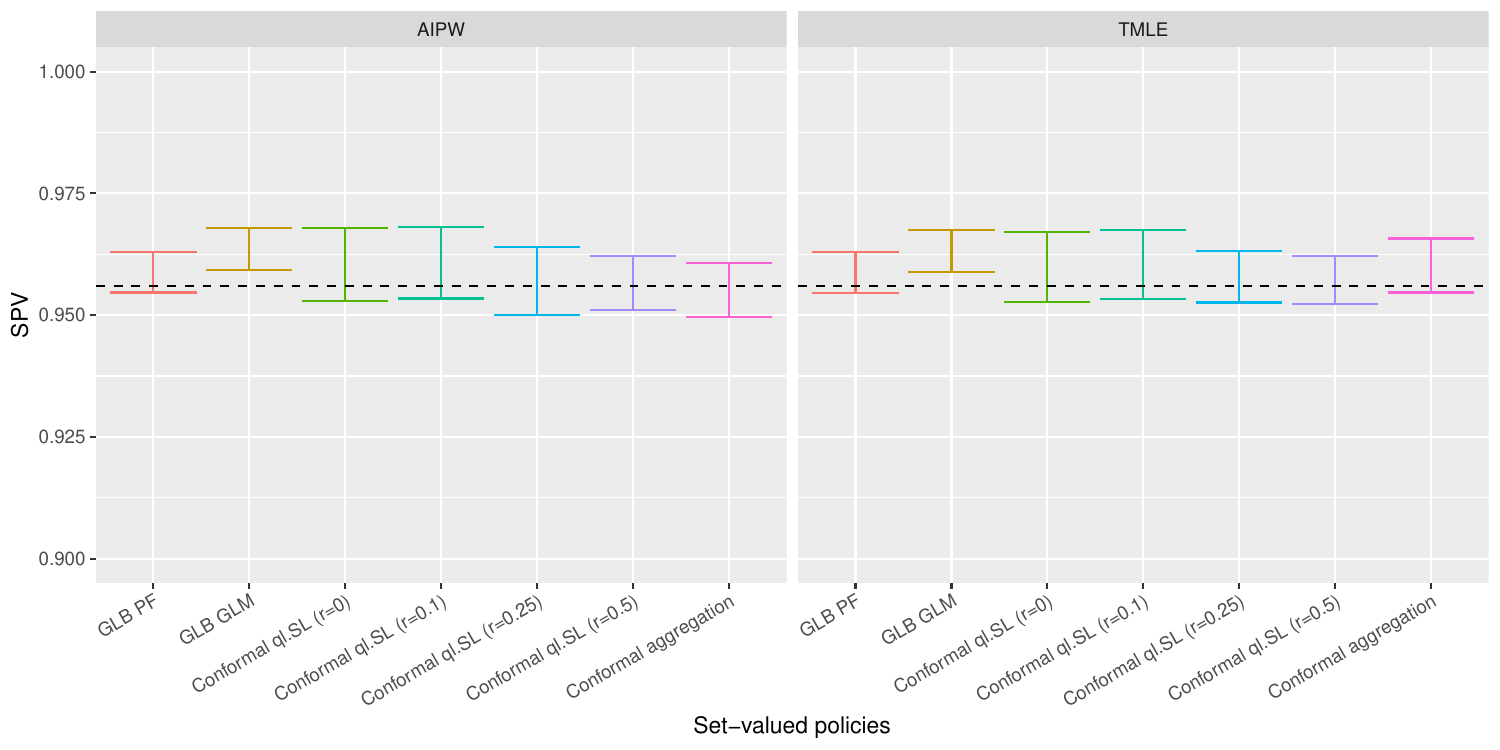}
    \caption{Asymptotic confidence intervals for the propensity AIPW (left) and TMLE (right) SPV estimators. Horizontal black dashed line indicates the clinician's policy value.}
    \label{fig:asymptotic:ci:traumacare}
\end{figure}

\subsection{Computational resources}
\label{app:computational:resources}
Simulations for the synthetic setting and the trauma care application were performed on a 10-core workstation, utilizing 4 CPU cores for parallel processing during computationally intensive tasks.
\end{document}